\documentclass[lettersize,journal]{IEEEtran}
\usepackage{amsmath,amsfonts}
\usepackage{algorithmic}
\usepackage{algorithm}
\usepackage{array}
\usepackage[caption=false,font=normalsize,labelfont=sf,textfont=sf]{subfig}
\usepackage{textcomp}
\usepackage{stfloats}
\usepackage{url}
\usepackage{verbatim}
\usepackage{graphicx}
\usepackage{booktabs}
\usepackage{amsthm}
\usepackage{cite}
\usepackage{xcolor}
\usepackage{xr}
\usepackage{hyperref}
\makeatletter
\newcommand*{\addFileDependency}[1]{%
  \typeout{(#1)}%
  \@addtofilelist{#1}%
  \IfFileExists{#1}{}{\typeout{No file #1.}}%
}
\makeatother

\newcommand*{\myexternaldocument}[1]{%
  \externaldocument[S-]{#1}%
  \addFileDependency{#1.tex}%
  \addFileDependency{#1.aux}%
}

\myexternaldocument{supplementary}

\newtheorem{theorem}{Theorem}

\begin{document}

\title{End-to-End Optimization of Incoherent Imaging for Classification Under Detector-Limited Readout}



\author{\IEEEauthorblockN{Archer Wang, Joshua Chen, Sachin Vaidya, Marin Soljačić}%
\thanks{Archer Wang, Joshua Chen, Sachin Vaidya, and Marin Solja\v{c}i\'{c}
are with the Research Laboratory of Electronics, Massachusetts Institute of
Technology, Cambridge, MA 02139 USA
(e-mail: archerdw@mit.edu; chenjosh@mit.edu; svaidya1@mit.edu; soljacic@mit.edu).}
\thanks{Marin Solja\v{c}i\'{c} is also with the Department of Physics,
Massachusetts Institute of Technology, Cambridge, MA 02139 USA.}%
}



\maketitle

\begin{abstract}
End-to-end co-optimization of optical front-ends (e.g. metasurfaces) and neural network back-ends has been widely applied to imaging tasks, yet a formalism characterizing when and why such systems outperform conventional lens-based imaging is largely lacking. This paper focuses on object classification, a central imaging task, and asks when end-to-end optimization of a phase mask for incoherent imaging improves performance over a conventional focusing lens. We find that these gains arise primarily under constrained detector readout and are limited under full detector readout. In the latter setting, we prove that no incoherent phase mask exceeds the ideal-channel mutual information between detector measurements and class labels; a conventional focusing lens approaches this ceiling, and joint optimization yields no empirical gain. When detector readout is constrained---by coarse spatial sampling or a limited number of measurements---optimized optics can substantially improve classification by increasing class separability in the detector measurements. These gains are largest under low detector noise and shrink as noise grows, because the optics shape the signal before it reaches the detector but cannot remove noise added afterward. The advantage also depends on the spectral structure of the task: co-design helps most when class-discriminative content is concentrated at lower spatial frequencies than within-class variation. We develop a theoretical framework formalizing these distinctions and test its predictions on synthetic data and standard benchmarks (MNIST, FashionMNIST, SVHN).
\end{abstract}


\section{Introduction}
\IEEEPARstart{O}{ptical} imaging systems are increasingly being
co-designed with downstream learning pipelines for inference tasks\cite{arya2024metasurfaces,Fisher2025,Molesky2018,Piggott2015,chen2026wavefront}. In this paradigm, an optical front-end---such as a diffractive stack, metasurface, or compound optic---is jointly optimized with a digital back-end to minimize a task-specific objective, often by rendering the image-formation model differentiable to enable gradient-based optimization. This departs from the conventional lens-based pipeline, in which a standard focusing lens projects an image onto the detector and subsequent processing is performed entirely in the digital domain. 

In many practical inference systems, recording and processing a dense, full-resolution detector image before classification is costly. Reading out a dense detector array, digitizing every pixel, and moving those measurements through the system can impose substantial latency, power, and bandwidth costs. In high-speed perception settings, these costs directly limit response time; in embedded and always-on platforms, they can dominate the energy budget; and in non-visible wavelength regimes, dense detector arrays may be costly or difficult to scale~\cite{baek2025edge,choi2025fsoe,gehrig2024event,gibson2020singlepixel,stantchev2020thz}. For example, a concrete application is high-throughput industrial inspection
and sorting, where classification systems must identify defects, materials, or
product categories on rapidly moving assembly lines or sorting
streams~\cite{malamas2003survey,golnabi2007industrial}. In settings like these, the system may be forced to operate with coarse spatial sampling or a limited number of detector measurements. Hence, the design objective should shift from producing a visually faithful image at the detector to ensuring that limited detector measurements preserve class-relevant information. A further consideration in these fast-vision systems is detector noise. Short exposure times limit the signal per measurement, making noise added at the detector non-negligible relative to the signal. The amount and spatial pattern of detector readout then become design choices, trading off signal capture against readout cost.

Much prior work on end-to-end optical design has focused on computational cameras, where optics and post-processing are jointly optimized for classical imaging objectives such as extended depth of field, chromatic correction, and super-resolution~\cite{Sitzmann2018,Tseng2021Compound,Colburn2018,Tseng2021Nano,Min2025MulticolorScintillators}. These works establish the broader paradigm of differentiable optical--digital co-design, but they do not directly address when such co-design improves downstream classification relative to a conventional focusing lens, particularly under detector readout constraints.

More directly relevant are works that target task-specific inference in the optical domain. Diffractive optical neural networks implement learned linear transformations via cascaded diffractive layers and can be trained end-to-end for classification or other inference tasks~\cite{Lin2018D2NN,Luo2019}. Related hybrid approaches place an optimized optical front-end before digital neural networks to accelerate inference or reduce energy consumption in early-layer computation~\cite{Colburn2019}. Collectively, these results suggest that when the objective is task performance rather than pixel fidelity, the optimal measurement operator can differ substantially from a conventional focusing lens.

This motivates the central question of this work: for object classification, under what conditions does optical--computational co-design provide a meaningful advantage over a conventional focusing lens pipeline, and how large can that advantage be?

\subsection{Contributions}
Recent work has examined modern imaging and optics pipelines 
through the lens of classical information theory\cite{pinkard2025information,kabuli2026lensless,markley2025efficient,Hamerly2025}. In that spirit, we study 
image classification under a range of detector readout schemes, 
spanning full readout to spatially aggregated and sparse 
measurements. Prior work has demonstrated that accurate 
classification is achievable using only a sparse subset of 
detector area\cite{jaeger1994sparsedatascan,mennel2022sparsepixel,brunton2013optimalsensorplacement,kutz2013datadriven}; however, these approaches do not jointly optimize the optical front-end or account for detector noise. Our work addresses 
both gaps.

Our main contributions are as follows:
\begin{itemize}
    \item We prove that under full detector readout, no passive incoherent phase mask exceeds the mutual information of the ideal imaging channel (Section~ \ref{sub:mi_proof}), which we approximate in our simulations with a conventional focusing lens.
    
    \item We derive a mathematical framework showing how the scene spatial-frequency content, optical transfer function, detector readout, and detector noise jointly determine class separability and classification performance. The analysis also predicts that co-design helps 
most when the class-mean difference is concentrated at 
lower spatial frequencies than the within-class 
variation, and we test this prediction with 
frequency-shifted synthetic controls.
    
    \item We validate these predictions through toy Gaussian-model experiments, frequency-shifted synthetic controls, and standard benchmark datasets, showing the largest gains under masked readout and low detector noise, and limited benefit under full detector readout, block-sum readout, or higher detector noise.
\end{itemize}

\begin{figure}[t]
\centering
\subfloat[]
\hfill
\includegraphics[width=0.9\linewidth]{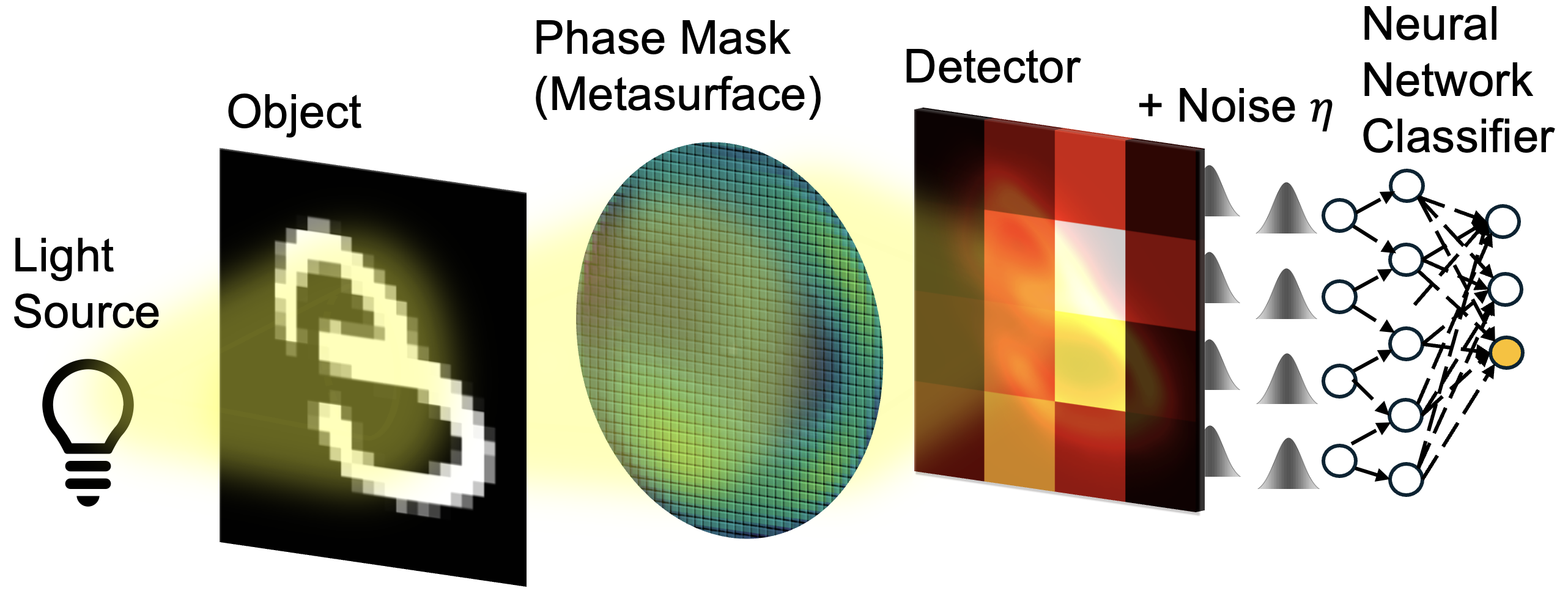}
\subfloat[]
\hfill
\includegraphics[width=0.9\linewidth]{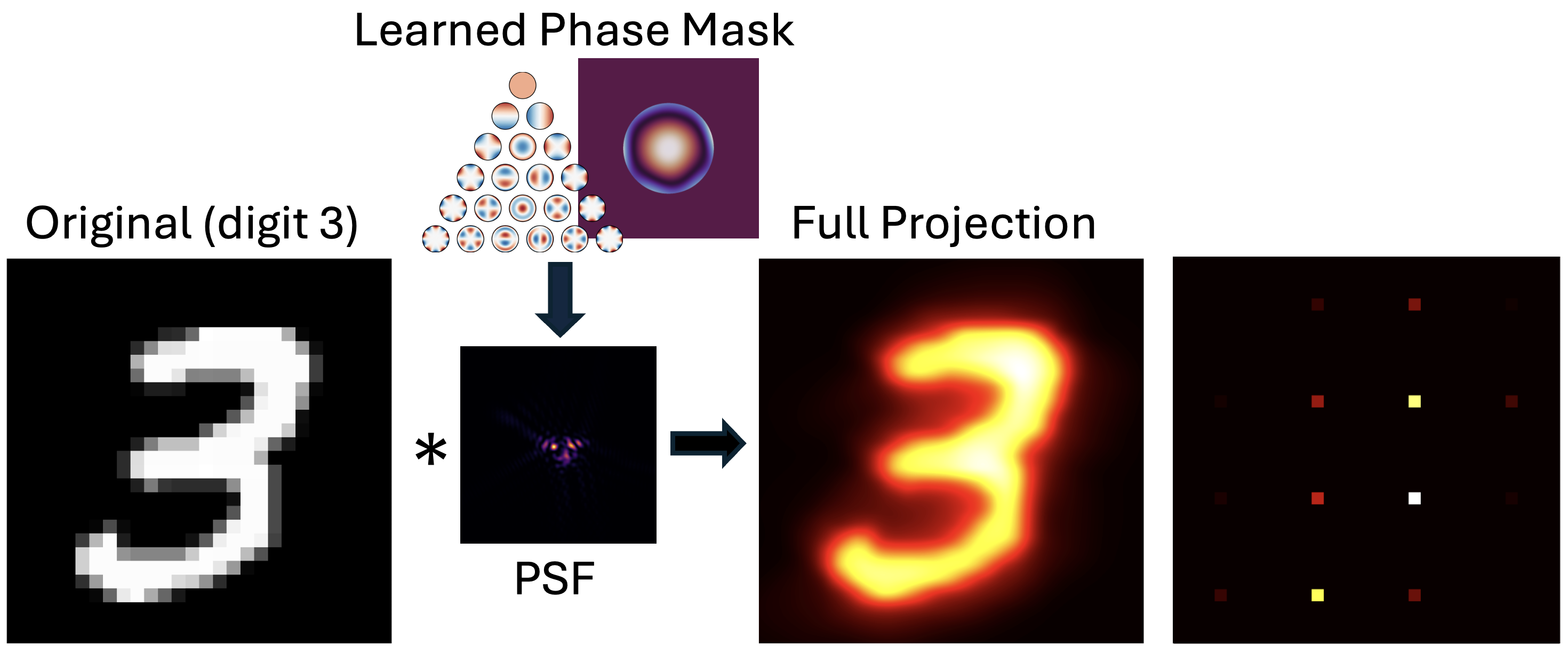}
\caption{Incoherent imaging and classification pipeline. (a) End-to-end system overview: an input object intensity is optically convolved with the system point spread function (PSF) induced by a learned phase mask, producing an intensity distribution at the detector. Measurement noise is then added, and the resulting readout vector is passed to a neural network classifier. The phase mask and classifier are trained jointly end-to-end. (b) Example of detector-limited readout: an input digit is convolved with the PSF corresponding to the learned phase mask, yielding the full detector-plane intensity, after which only a subset of detector-region integrals is recorded and used for classification.}
\label{fig:e2e_diagram}
\end{figure}

\section{Imaging Process}

\subsection{Incoherent Imaging}
\label{sec:imaging_models}
We use a standard scalar incoherent imaging model for a monochromatic, shift-invariant imaging system. This captures the class of conventional
local-response phase masks and metasurfaces considered in this work, where
the optical element is modeled as a spatially varying complex transmission
function. In this model, each object point produces the same shifted
point-spread function (PSF) at the detector. For spatially incoherent
illumination, intensities from different object points add, so image formation
is modeled as a convolution between the object intensity $x(\mathbf r)\ge 0$ and an
intensity PSF $h(\mathbf r)\ge 0$, with
$\int h(\mathbf r)\,d\mathbf r \le 1$. The imaging equation is then
\begin{equation*}
y(\mathbf r) \;=\; (h * x)(\mathbf r) \;+\; \epsilon(\mathbf r),
\end{equation*}
where $y(\mathbf r)$ is the detected intensity and $\epsilon$ is the noise associated with the detection process~\cite{Goodman2017Fourier}. In the Fourier domain, $
\widehat{y}(\boldsymbol{\omega}) = \widehat{h}(\boldsymbol{\omega})\,\widehat{x}(\boldsymbol{\omega})
+ \widehat{\epsilon}(\boldsymbol{\omega})$,
where $\widehat{h}$ denotes the optical transfer function (OTF). Consequently, $\widehat{h}(0) \le 1$ and
$|\widehat{h}(\boldsymbol{\omega})| \leq 1$ for all $\boldsymbol{\omega}$
\cite{Goodman2017Fourier,BornWolf1999}.

\begin{figure}[t]
\centering
\includegraphics[width=0.9\linewidth]{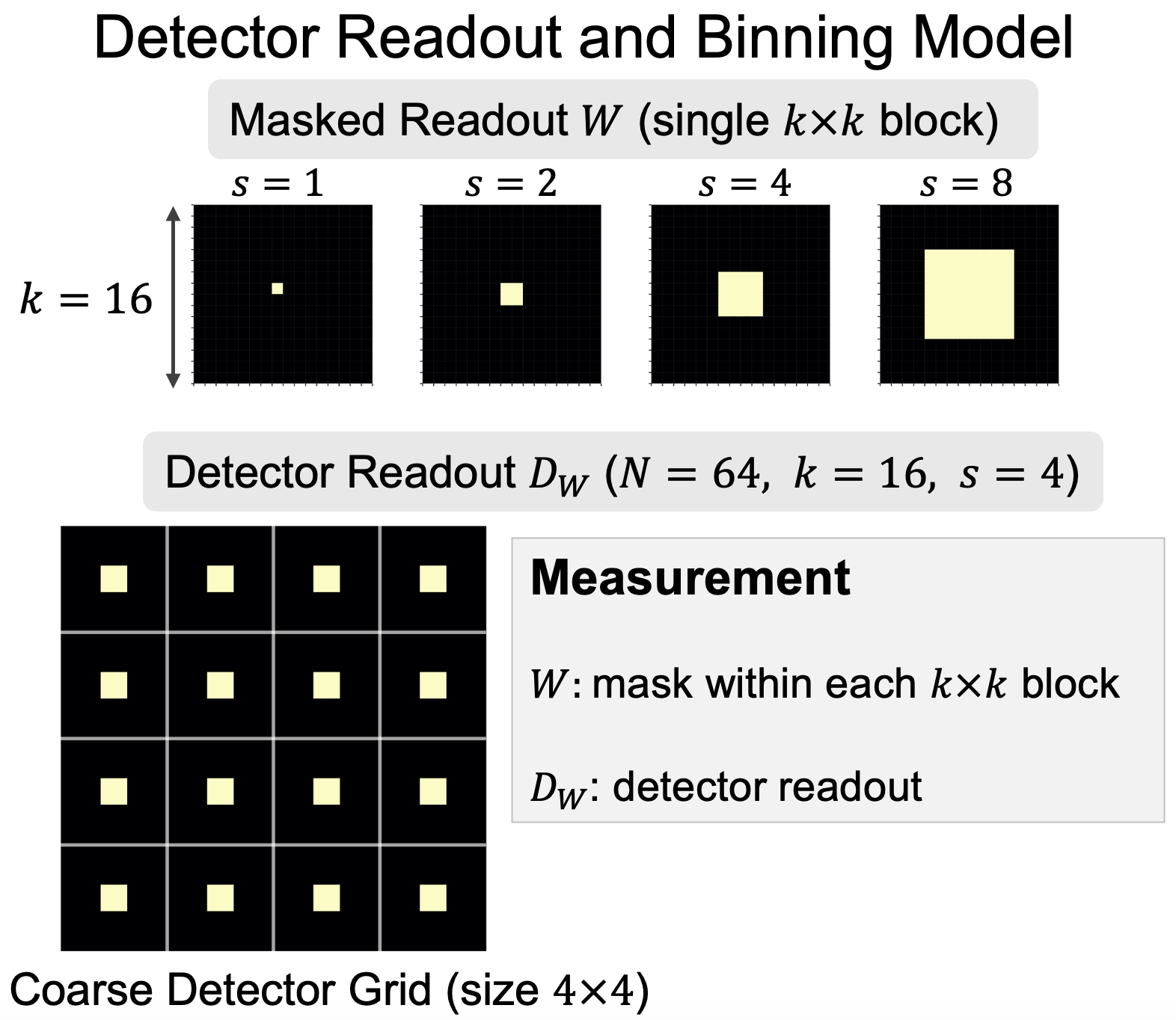}
\caption{Detector divided into $k\times k$ blocks, which may correspond to physical pixel regions. Within each block, the pale yellow square denotes a readout region of side length ($s$); intensities in this region are summed to produce one measurement per block. The top row shows the mask ($W$) on a single block for different ($s$). The bottom row shows the resulting detector readout ($D_W$) applied across the full detector.
}
\label{fig:detector}
\end{figure}

\subsection{Detector Readout and Noise Model}
\label{sec:detector_binning}

This section defines the detector model used throughout the paper, including full readout and the block-sum, sparse, and masked configurations that parameterize our detector-limited regimes. We discretize the continuous image formation model onto a sufficiently fine detector-plane grid, yielding $Y_{\mathrm{fine}}\in\mathbb{R}^{N\times N}$, with $N$ chosen so that discretization error is negligible relative to the detector-limited readout effects studied below. $Y_{\mathrm{fine}}$ is the intensity pattern formed after light passes through the phase mask and propagates to the detector plane prior to detector noise. When the detector records one measurement at each \emph{fine-grid} location, we refer to this as \textbf{full readout}.

To model larger detector pixels, we partition this fine grid into disjoint $k\times k$ blocks and aggregate the intensity within each block. Varying $k$ corresponds to varying the physical pixel size. The
resulting binned detector output lives on a \emph{coarse} grid of size
$N' \times N'$, where $N' := N/k$.

The readout is a weighted sum over each block with 
mask $W$; we write the operator as $D_W$ and give its 
explicit form in Appendix~\ref{app:detector-operator}. When $W \equiv 1$, the detector performs \textbf{block-sum readout}, aggregating all pixels within each block. A one-hot $W$ corresponds to \textbf{sparse readout}, where only a single pixel per block is measured. More generally, a centered $s \times s$ indicator (with $s \le k$) defines \textbf{masked readout}, which aggregates a local subset of pixels within each block. In all cases, the readout produces one
measurement per block, indexed on the coarse detector grid. We refer to readout configurations that do not record all entries of the fine-grid detector image---including coarse, sparse, and masked readout---as \emph{detector-limited} regimes. Fig.~\ref{fig:detector} illustrates this readout model.

Finally, noise associated with the detection process is modeled as additive
white Gaussian noise, with entries $
\epsilon[p,q] \overset{\mathrm{i.i.d.}}{\sim} \mathcal N(0,\sigma^2)$.
The final measurement is $
Y = D_W Y_{\mathrm{fine}} + \epsilon$.

\section{Mutual Information Analysis and End-to-End Optics Modeling}

\subsection{Classification Objective}
To analyze how optical design affects classification performance, we model the joint imaging–inference chain as a Markov 
process $C \to X \to Y$, where $C$ is the class label, 
$X$ the object intensity, and $Y$ the detector 
measurement. In experiments, $Y$ is passed to a neural 
network trained end-to-end with the optical front end 
under a cross-entropy objective. For theoretical 
analysis we assume a Bayes-optimal classifier on $Y$ 
and use the mutual information $I(C; Y)$ as our 
analysis proxy~\cite{Ding2022}: higher $I(C; Y)$ implies a lower 
achievable error probability via Fano's inequality 
(Supp. Section~SII). The Bayes-optimal 
analysis therefore provides a theoretical ceiling for 
the trained networks.

With this classification objective established, we ask 
whether the optical front-end can increase the mutual 
information between detector measurements and class labels 
beyond what a conventional focusing lens achieves.

\subsection{Full Readout: Optics Cannot Increase Mutual Information}
\label{sub:mi_proof}
Here, we show that under full readout, the ideal imaging channel is information-theoretically optimal among passive incoherent phase masks, and that a diffraction-limited focusing lens approaches this bound. We now consider the full-readout setting introduced in Section~\ref{sec:detector_binning}, in which the discretized detector image is recorded at all fine-grid locations. Here, \(\|\cdot\|_2\) denotes the \(\ell_2\) norm, and for an operator \(A\), \(\|A\|_2\) denotes the induced operator norm. For any phase mask, the corresponding discrete PSF satisfies \(h[n]\ge 0\) and \(\sum_n h[n]\le 1\). Convolution with \(h\) is contractive on \(\ell_2\) norm (see Appendix~\ref{app:theorem1-proof}); equivalently, the imaging operator \(A\) satisfies \(\|A\|_2 \le 1\).
This contractivity of $A$ is the key property used below.

\begin{theorem}
\label{thm:mi}
Consider the discrete full-readout imaging model
\[
Y = AX + \epsilon, \qquad \epsilon \sim \mathcal{N}(0,\sigma^2 I),
\]
where \(A\) is a contractive linear operator, i.e.\ \(\|A\|_2 \le 1\). Let $A = I$ denote the ideal imaging channel. Then
\[
I(C;Y) \le I(C;X+\epsilon).
\]
\end{theorem}
In other words, a contractive linear map followed by isotropic Gaussian noise is a degraded version of the identity map followed by isotropic Gaussian noise, and therefore cannot increase mutual information. Concretely, one can decompose the noise into a component passed through $A$ and an independent remainder, yielding a Markov chain
$C \to X \to X + \eta \to A(X + \eta) \to Y$ so the result follows from the data-processing inequality. A complete proof can be found in Appendix~\ref{app:theorem1-proof}. Crucially, under full detector readout, no passive incoherent phase mask exceeds the ideal-channel mutual information bound; a conventional focusing lens approaches this bound in practice. By contrast, when the detector performs coarse or partial readout, the recorded measurements no longer preserve the full detector-plane intensity, and a conventional focusing lens need not be optimal for classification; accordingly, we focus on detector-limited regimes in which optimized phase-mask optics can improve downstream performance.

\subsection{Setup}
\label{sub:gaussian_setup}
To quantify the benefit of optical optimization under 
detector binning or partial readout, we adopt a shared-covariance Gaussian model,
which has closed-form expressions in later analysis. For analytical clarity, we present the derivations in one
spatial dimension. The extension to two-dimensional imaging is
straightforward by applying the same arguments to each
two-dimensional spatial-frequency index.

Let \(C \in \{1,\dots,K\}\) with priors \(\{\pi_c\}\), and for each class \(c\), let $
\mu_c := \mathbb{E}[X \mid C = c]$.
Assume
\[
X \mid C = c \;\overset{d}{=}\; \mu_c + Z, \qquad Z \sim \mathcal{N}(0, \Sigma),
\]
where \(\Sigma\) is the covariance matrix of \(Z\). We assume that \(\Sigma\) is induced by a wide-sense stationary process and hence is diagonalized by the discrete Fourier basis. Writing
\(\mathcal F\) for the set of discrete Fourier transform (DFT) frequency bins, let \(\widehat{S}[\ell]\in\mathbb{R}\), \(\ell \in \mathcal F\) denote the corresponding diagonal entries.
This Fourier-domain diagonalization enables a frequency-by-frequency characterization. Such stationary covariance models are
standard in detection and estimation theory for representing
correlated backgrounds, noise, or clutter \cite{Kay1998Detection,VanTrees2001Detection,Abbey2007}. They are also commonly used in image modeling; for example, in classical models of natural images exhibiting a $1/f$ decay in spatial frequencies, the covariance operator is diagonal in the Fourier domain, with eigenvalues that decrease inversely with frequency \cite{vanDerSchaaf1996,Yedidia2018Aperture}. Furthermore, prior work suggests that spatial frequency is tied to semantic scale: between-class or category-level structure is often concentrated in low- and mid-spatial frequencies, whereas within-class or subordinate-level variation can rely more strongly on higher spatial frequencies~\cite{oliva2001spatialenvelope,torralba2003naturalcategories,collin2005subordinate}. This motivates our modeling focus on the relative frequency distributions of class-mean differences and within-class variability, which we instantiate in the synthetic setting and empirically examine on benchmark datasets in Figs.~\ref{fig:radial_mtf}(b) and~\ref{fig:additional_datasets}(b)--(d).

For a discrete signal \(u\), let \(\widehat u[\ell]\) denote its DFT
coefficient at frequency bin \(\ell \in \mathcal F\). Let \(A\) denote an incoherent linear shift-invariant imaging
system acting on intensity, with discrete transfer function
\(\widehat{H}[\ell]\) satisfying $|\widehat{H}[\ell]| \le 1$ and $\widehat{H}[0] = 1$. Let $D_W$ denote the detector readout operator introduced in Section~\ref{sec:detector_binning}. The measurement 
model is then $Y = D_W(h \ast X) + \epsilon$, where \(\epsilon \sim \mathcal N(0,\sigma^2 I)\)
models detector noise on the coarse detector grid.

\section{Binary Classification}

\subsection{Separability}

Using the model of
Section~\ref{sub:gaussian_setup}, we consider binary classification (\(K=2\)) with equiprobable classes and shared covariance. After imaging and detector readout, the two class-conditional measurement distributions remain Gaussian with shared covariance, $Y \mid C = c \sim \mathcal{N}(\mu_{Y,c}, 
\Sigma_Y)$ for $c \in \{0, 1\}$. The natural separability measure is the squared 
Mahalanobis distance between the class means,
\begin{equation}
d^2 := \Delta \mu_Y^\top \Sigma_Y^{-1}\Delta \mu_Y.
\label{eq:d2_def}
\end{equation}
where $\Delta \mu_Y:=\mu_{Y,1} - \mu_{Y,0}$. Intuitively, $d^2$  generalizes the 1D notion of how many standard deviations apart the class means are to the multivariate setting, accounting for within-class covariance. This is useful because in the equal-covariance Gaussian setting, both the mutual information \(I(C;Y)\) and the Bayes classification error depend on the observation model only through \(d^2\), making it a sufficient statistic for classification performance. Under this Gaussian model, \(I(C;Y)\) is monotonically increasing in \(d^2\) (Supp. Section~SII, Thm. 1). The Bayes-optimal classifier in this setting is linear, as in linear discriminant analysis (LDA). LDA and related Gaussian discriminant models\footnote{When the equal-covariance assumption is relaxed, more general separability measures such as the Bhattacharyya or Chernoff distances may be used instead~\cite{Goudail2004,Nielsen2014}; LDA-style assumptions are also known to be robust under mild mismatch~\cite{lachenbruch1975,klecka1980}.} are classical tools in discrimination tasks, including face recognition and other pattern-classification problems~\cite{bishop2006pattern,etemad1997discriminant,gan2014selftraining}.

Under the stationarity assumption, the Fourier basis diagonalizes the covariance, enabling a frequency-by-frequency decomposition of $d^2$ as a ratio of between-class signal power to within-class variability---a spectral viewpoint that becomes particularly informative once detector constraints introduce aliasing. When frequencies remain decoupled, $d^2 = \sum_{\ell\in\mathcal F} \widehat{d}^2[\ell]$,
where each term \(\widehat{d}^2[\ell]\) quantifies the contribution of frequency bin \(\ell\) to overall separability. Detector-limited readout breaks this decoupling: aliasing folds multiple fine-grid frequencies into each coarse-grid bin, modifying the additive decomposition. Appendix~\ref{app:exact_alias} derives the corresponding expressions.

\subsection{Synthetic Binary Classification Experiment}
\label{sec:toy_binary}

We perform a binary classification experiment using a 
physical imaging model (bandlimited angular-spectrum propagation with a finite pupil-aperture mask applied at the lens plane \cite{Matsushima2009BLAS}; 
see Appendix~\ref{app:toy_binary_details}) with equal class priors 
and within-class covariances. The imaging region spans
\(1\,\mathrm{mm}\times1\,\mathrm{mm}\), with \(1\times\) magnification between the
object and detector planes. To isolate the effects of the optics and detector sampling on class separability, we use a synthetic Gaussian model. Specifically, an object image
$X\in\mathbb{R}^{N\times N}$ is drawn from one of two equally likely classes sharing a common covariance, $X\mid C=c\sim\mathcal{N}(\mu_c,\Sigma)$,
where $\mu_0$ and $\mu_1$ differ through a localized two-lobed pattern
formed by subtracting one spatial Gaussian from another at a different nearby location and adding the result to a nonnegative background (see Fig.~\ref{fig:binary_sweeps}a)), and $\Sigma$ is the within-class covariance of a standard wide-sense stationary process whose power spectral density is a radially symmetric Gaussian. This construction instantiates the regime discussed in Section~\ref{sub:gaussian_setup}, with low-frequency class-mean differences and relatively higher-frequency within-class variability. The deliberately separated spectra isolate the detector-limited mechanism and make the optical filtering and aliasing effects easier to visualize.

Measurements are generated via an incoherent, shift-invariant imaging model, with detector-limited readout. We then jointly optimize the phase mask and a logistic regression classifier, which is Bayes-optimal under the equal-covariance Gaussian model. Finally, we evaluate both the squared Mahalanobis distance and the test accuracy of the end-to-end optical system. Appendix~\ref{app:toy_binary_details} gives the implementation details needed to reproduce the experiment, including the physical propagation model, phase-mask parameterization, detector readout configurations, within-class covariance model, and training hyperparameters.

We study two detector-limited regimes. The fine-grid resolution is $N=256$. First, in a block-sum readout
sweep, we vary the detector resolution
$N'\in\{2,4,8,16,32\}$, where $N'=N/k$. Second, in a
masked-readout sweep, we fix $N'=16$ and vary the readout budget $s$ by summing a centered $s\times s$ region within
each $k\times k$ block, for $s\in\{1,2,4,8,16\}$. In both regimes, we compare a fixed conventional lens against joint
optimization of the optical phase and a linear classifier, and report
both $d^2$ (computed using the detector-limited frequency-domain expression derived in Section~\ref{sec:detector_configs}) and validation accuracy. We first perform sweeps in the absence of detector noise.

\begin{figure}[t]
    \centering
    \subfloat[]{%
        \includegraphics[width=0.8\linewidth]{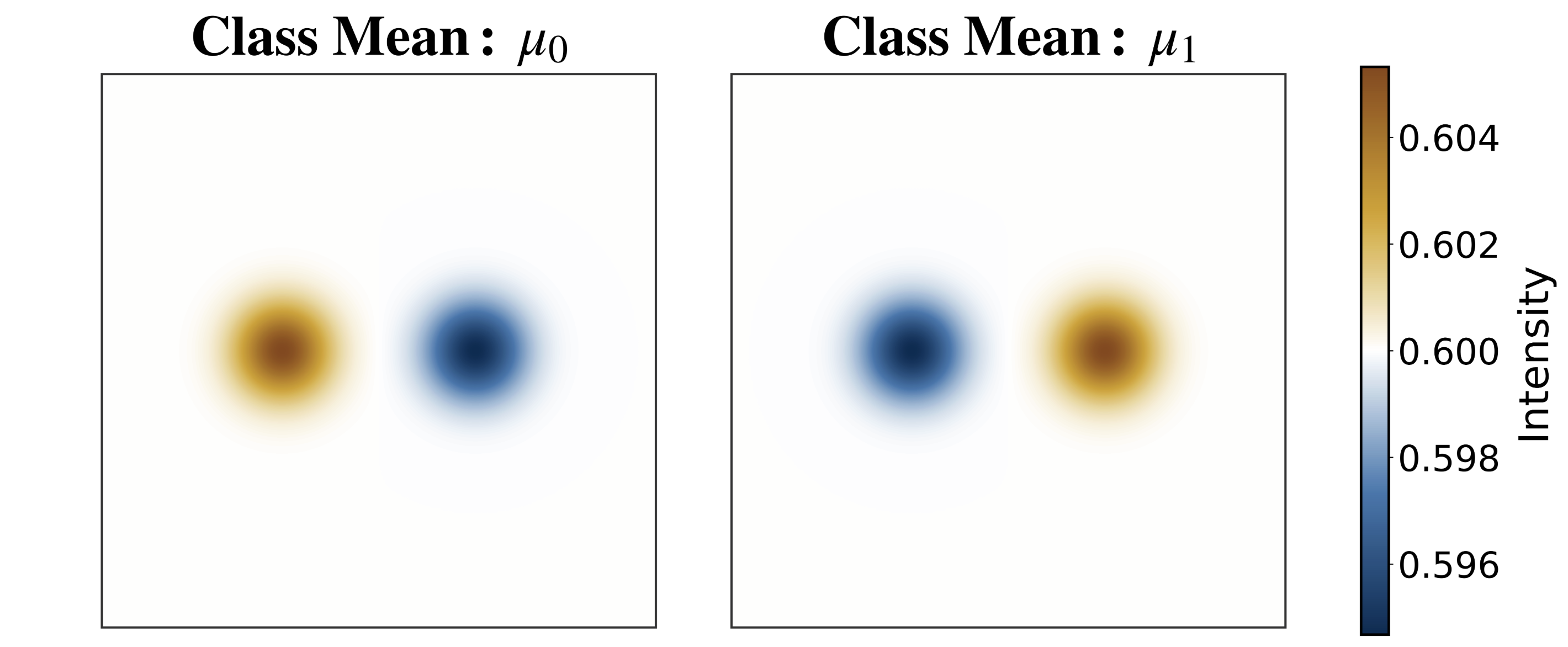}
        \label{fig:sub0}
    }
    \hfil
    \subfloat[]{%
        \includegraphics[width=0.9\linewidth]{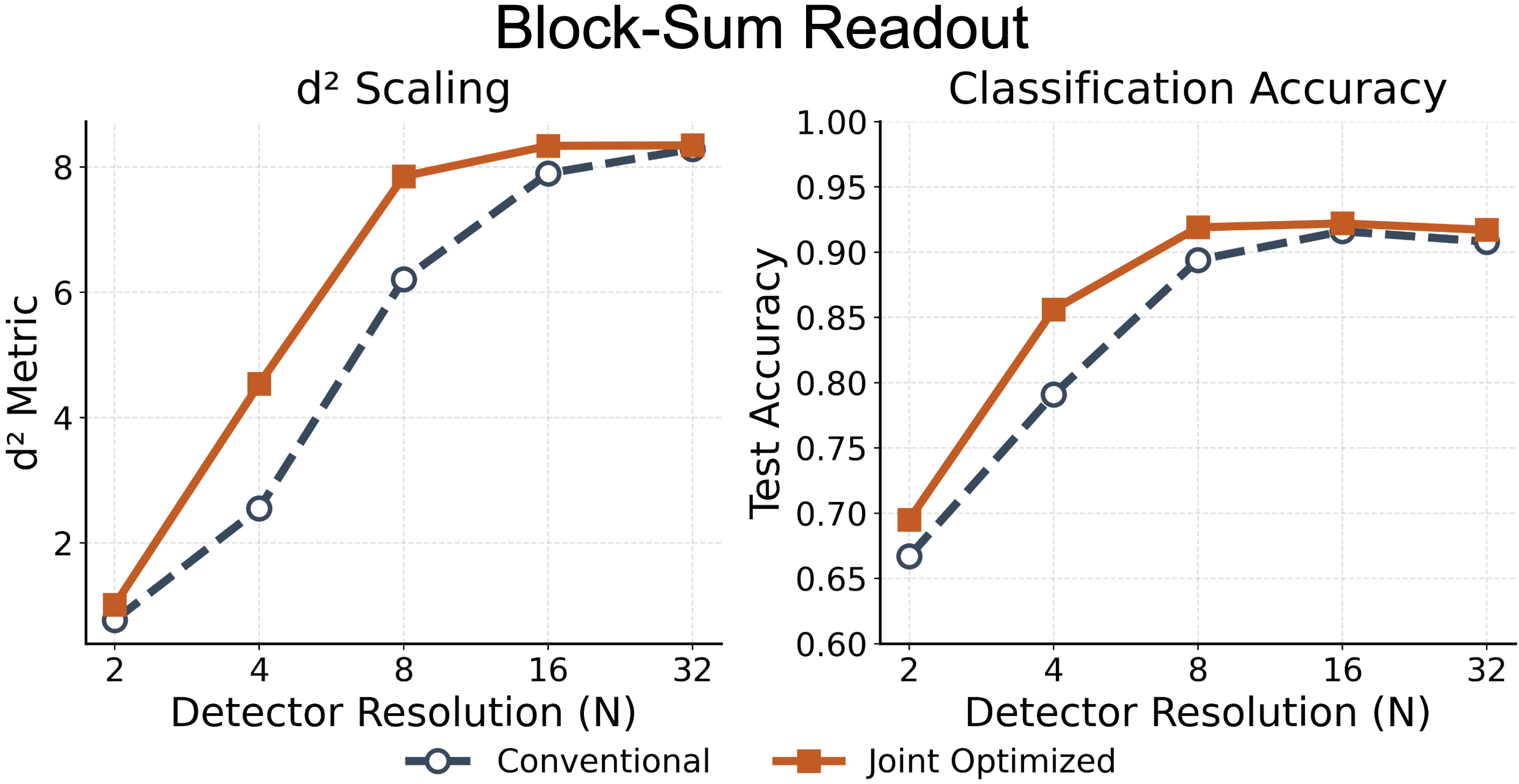}
        \label{fig:sub1}
    }
    \hfil
    \subfloat[]{%
        \includegraphics[width=0.9\linewidth]{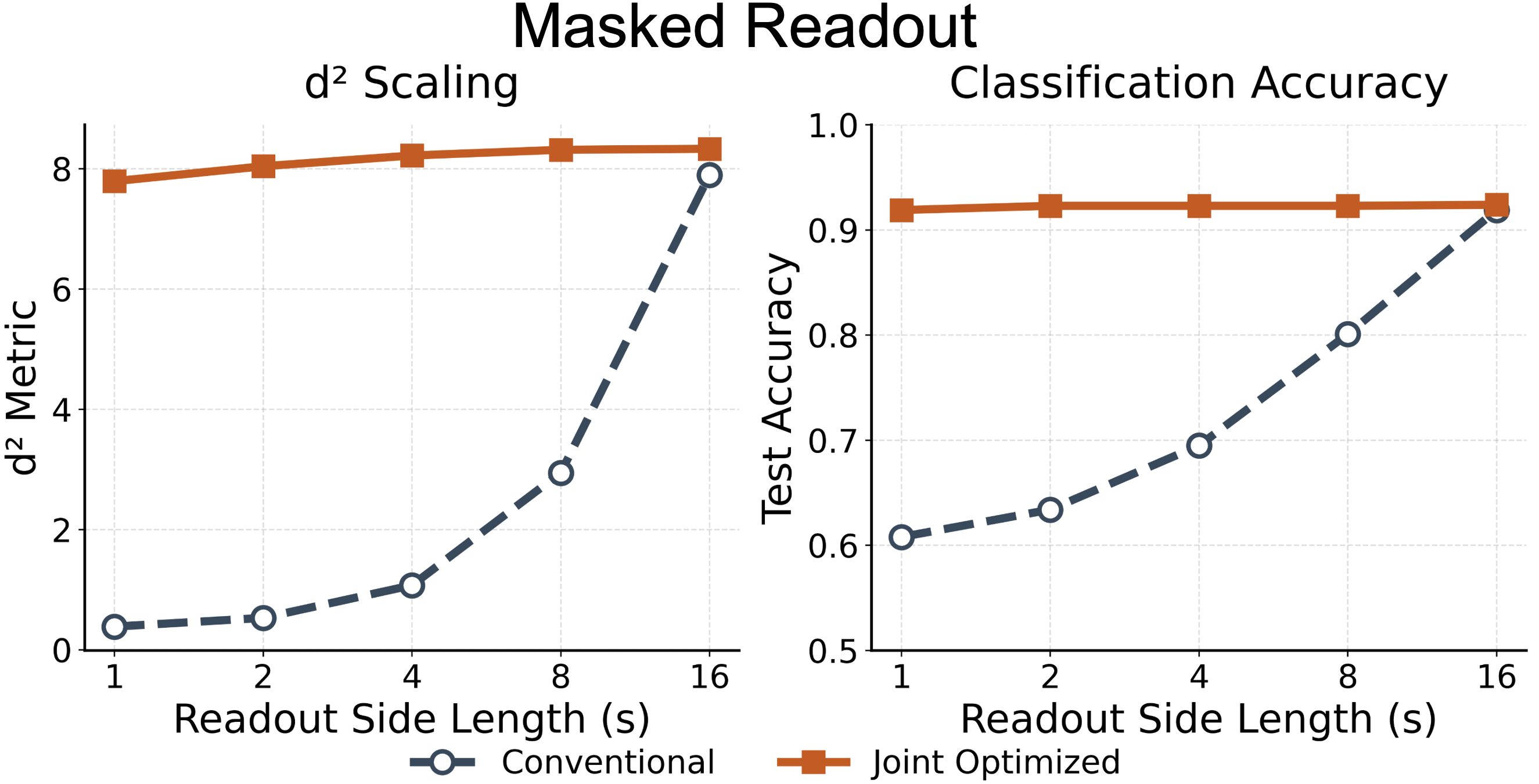}
        \label{fig:sub2}
    }
    \hfil
    \subfloat[]{%
        \includegraphics[width=0.9\linewidth]{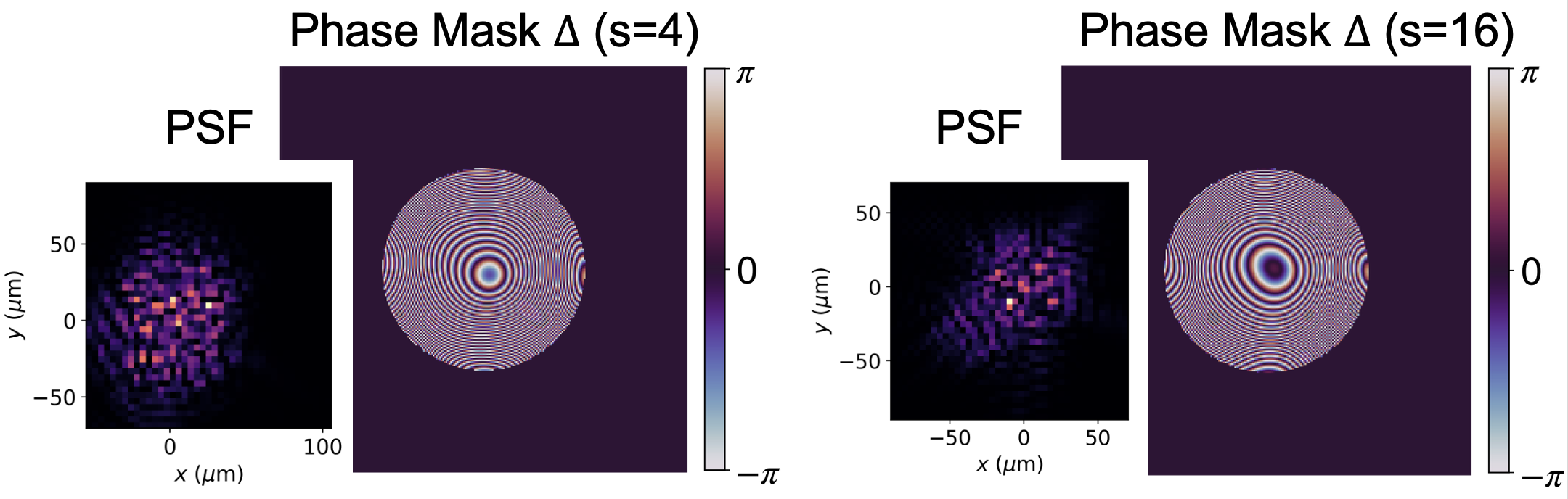}
        \label{fig:sub2}
    }
    \caption{\textbf{Binary classification under detector-limited readout (noiseless).}
(a) Class means \(\mu_0\) and \(\mu_1\) in the synthetic Gaussian experiment.
(b) \emph{Block-sum readout sweep:} detector resolution \(N'\) is varied, with each detector pixel summing a disjoint \(k\times k\) block. Left: separability proxy \(d^2\). Right: test accuracy.
(c) \emph{Masked-readout sweep:} detector resolution is fixed and the within-block readout size \(s\times s\) is varied. Left: separability proxy \(d^2\). Right: test accuracy. In both sweeps, jointly optimizing the optical phase and linear classifier improves performance relative to a conventional lens under detector constraints. (d) Phase mask difference $\Delta$ (optimized minus lens phase profile) and resulting PSFs for $s=4, 16$ masked readout. The phase difference is plotted to clearly visualize the learned deviations from the conventional lens profile.}
    \label{fig:binary_sweeps}
\end{figure}

In the block-sum readout sweep (Fig.~\ref{fig:binary_sweeps}b), the absolute performance of both the conventional lens and the jointly optimized phase-mask system degrades substantially as detector resolution decreases, indicating that coarse spatial aggregation removes task-relevant information for either optical design. The optimized system nevertheless maintains consistently better absolute performance than the conventional baseline, with the largest performance gap appearing at intermediate detector resolutions. This suggests that the learned optical front-end redistributes class-discriminative structure into measurements that are comparatively more robust to coarse readout, even though it cannot prevent the strong absolute degradation caused by severe spatial coarse graining. At very low detector resolution, both systems become capped by the loss induced by aggressive block summation, so the achievable benefit of optical optimization is limited. At high detector resolution, the gap also narrows as the system approaches the near full-readout regime, where additional optical optimization yields diminishing returns, consistent with Theorem~\ref{thm:mi}.

In the masked-readout regime (Fig.~\ref{fig:binary_sweeps}c), the absolute performance of the conventional lens baseline improves substantially as the readout side length \(s\) increases, since larger readout regions recover more useful spatial information. The jointly optimized system maintains high separability and accuracy even at small $s$: the learned phase mask effectively redirects the class-discriminative parts of the image onto the pixels the detector actually reads, rather than relying on the detector to capture them by chance. In all, the performance gap is largest at small $s$ and decreases as $s$ increases---once $s$ is large enough that the detector reads most of the image anyway, this redirection matters less. These trends extend to multiclass classification: a four-class synthetic 
experiment using the minimum pairwise Mahalanobis distance as the 
separability proxy exhibits the same qualitative behavior across both 
detector-limited regimes; see Supp. Section~SI.

\begin{figure}[t]
    \centering
    \includegraphics[width=\linewidth]{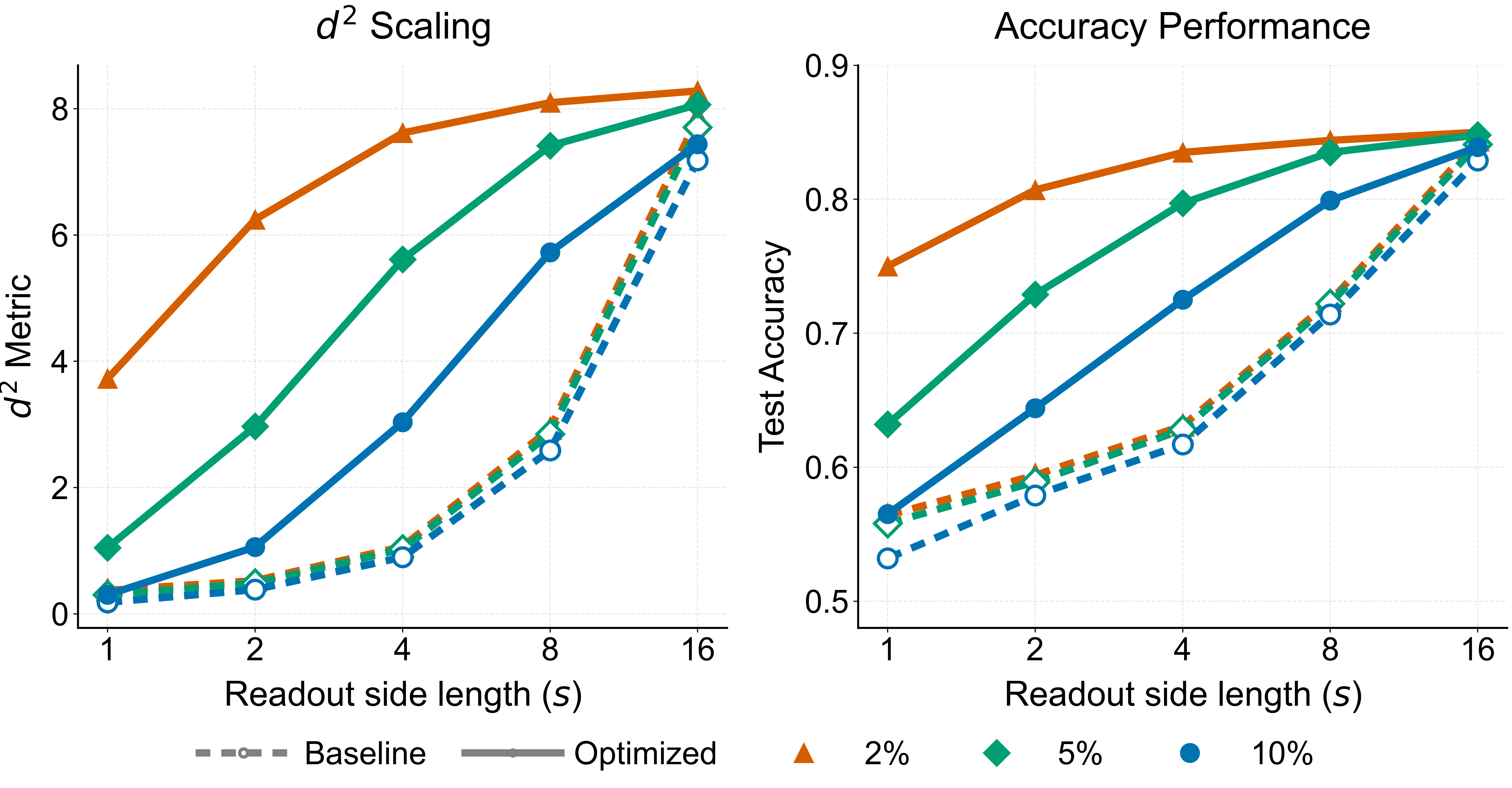}
    \caption{Effect of detector noise on binary masked-readout classification. (a) Separability metric \(d^2\) versus within-block readout side length \(s\)
for conventional and jointly optimized systems at three detector-noise levels.
(b) Corresponding test accuracy versus \(s\).
Increasing detector noise reduces the advantage of joint optimization, especially
in the sparse-readout regime (small \(s\)).}
    \label{fig:binary_noise}
\end{figure}

\subsection{Effect of Detector Noise}
\label{sub:binary_noise}

We next study how detector noise affects the absolute performance of the conventional lens baseline, the absolute performance of the jointly optimized system, and the gap between them under masked detector readout. To compare different readout sizes fairly, we fix the detector-noise variance per unit sensor area and evaluate all designs at \(N' = 16\), so that the noise variance of each coarse measurement scales with the area of the readout region. This normalization is broadly consistent with photon-counting noise models, in which shot-noise variance scales with the expected collected signal~\cite{Janesick2001ScientificCCD}. We run end-to-end optimization at three detector-noise levels and compare against a conventional focusing lens baseline. For each noise level, we scaled the noise at each fine-grid pixel such that the noise standard deviation was \{2\%, 5\%, 10\%\} of the mean object intensity per pixel.

Fig.~\ref{fig:binary_noise}a--b shows that, as detector noise increases, the absolute performance of both the conventional and jointly optimized systems decreases. The performance gap between them also shrinks, especially at small \(s\), where the readout is most constrained and the measurements are most sensitive to detector noise. At lower noise levels, the gap is largest at intermediate or small \(s\), where readout remains constrained enough for optical co-design to help, but the measurement SNR is not yet too severely degraded. Essentially, the optics can reshape which spatial frequencies reach the detector, but cannot suppress noise added after detection; as the detector noise floor grows, this added noise increasingly limits performance regardless of how the optics shape the signal, leaving less headroom for optical co-design to exploit.

\section{Analysis: Detector-Limited Readout Configurations}
\label{sec:detector_configs}
We now analyze the preceding observations using the mathematical structure of the $d^2$ separability expression. 

Consider a signal of length $N$, with downsampling factor $k$
such that $k \mid N$. Let $m \in \{0,1,\dots,N/k-1\}$ index the DFT bins of the downsampled signal. Under downsampling by a factor $k$, each coarse-grid bin $m$ receives contributions from the $k$ fine-grid DFT bins whose indices are congruent to $m$ modulo $N/k$. We denote these by
\[
\ell_r := m + r\frac{N}{k}, \qquad r=0,1,\dots,k-1.
\]
Let $\widehat{H}[\ell]$ denote the optical transfer function (OTF), i.e., the Fourier transform of the PSF $h$, sampled at the fine-grid DFT frequencies. For the detector readout operator $D_W$ (Section~\ref{sec:detector_binning}), let $\widehat{q_r}[m]$ denote the frequency-domain weight on the $r$th aliased branch $\ell_r$ in coarse-grid bin $m$. Explicit expressions for $\widehat{q_r}[m]$ are given below for the block-sum and sparse readout cases.

The contribution to the squared Mahalanobis distance at coarse-grid bin $m$ is
\begin{equation}
\widehat{d}^2[m]
=
\frac{\left|\sum_{r=0}^{k-1} \widehat{q_r}[m]\,\widehat{H}[\ell_r]\,\Delta\widehat{\mu}[\ell_r]\right|^2}
{\sum_{r=0}^{k-1}|\widehat{q_r}[m]|^2|\widehat{H}[\ell_r]|^2\widehat{S}[\ell_r]+k^2\sigma^2}
\label{eq:sectionV_exact_d2_dft}
\end{equation}
where $\Delta\widehat{\mu}[\ell]:=\widehat{\mu_1}[\ell]-\widehat{\mu_0}[\ell]$. The derivation, together with 
explicit expressions for $\widehat{q}_r[m]$ under 
block-sum, sparse, and masked readout, is given in Appendix~\ref{app:exact_alias_channel}. Throughout this section, we interpret 
Eq.~\eqref{eq:sectionV_exact_d2_dft} in the regime where 
the class-mean difference $\Delta\widehat{\mu}[\ell]$ is 
concentrated at low spatial frequencies while the 
within-class spectrum $\widehat{S}[\ell]$ has substantial 
energy at higher frequencies. This matches the synthetic 
setting of Section~\ref{sec:toy_binary}. We revisit this assumption 
in Section~\ref{sec:mtf_connection}, where the 
frequency-shifted controls of Fig.~\ref{fig:mid-freq} move the class-mean difference into mid- and 
higher-frequency bands while keeping the within-class covariance fixed.

For masked readout with window side length $s$, the 
detector response $\widehat{q}_{r}[m]$ is the DFT of 
a contiguous $s$-pixel window, which concentrates near 
zero frequency as $s$ grows (Appendix~\ref{app:exact_endpoints}). 
Increasing $s$ therefore suppresses higher-frequency 
aliased branches in both the numerator and denominator of Eq.~\eqref{eq:sectionV_exact_d2_dft}, making the 
coarse measurement progressively more low-pass before 
downsampling. The endpoint $s = k$ recovers block-sum 
readout. 


\begin{figure}[t]
\centering
\subfloat[]{\includegraphics[width=3.2in]{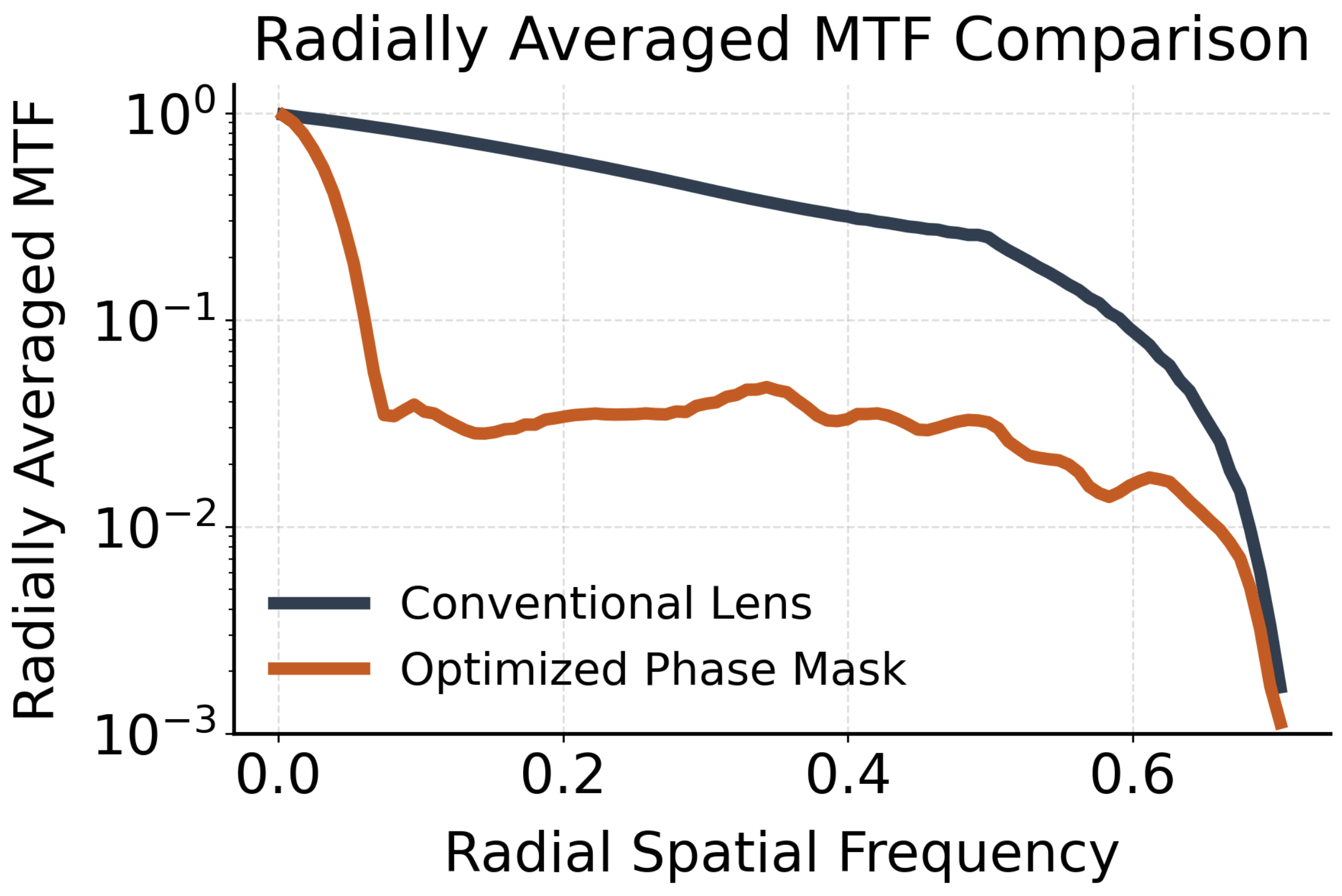}%
\label{fig_first_case}}
\hfil
\subfloat[]{\includegraphics[width=3.2in]{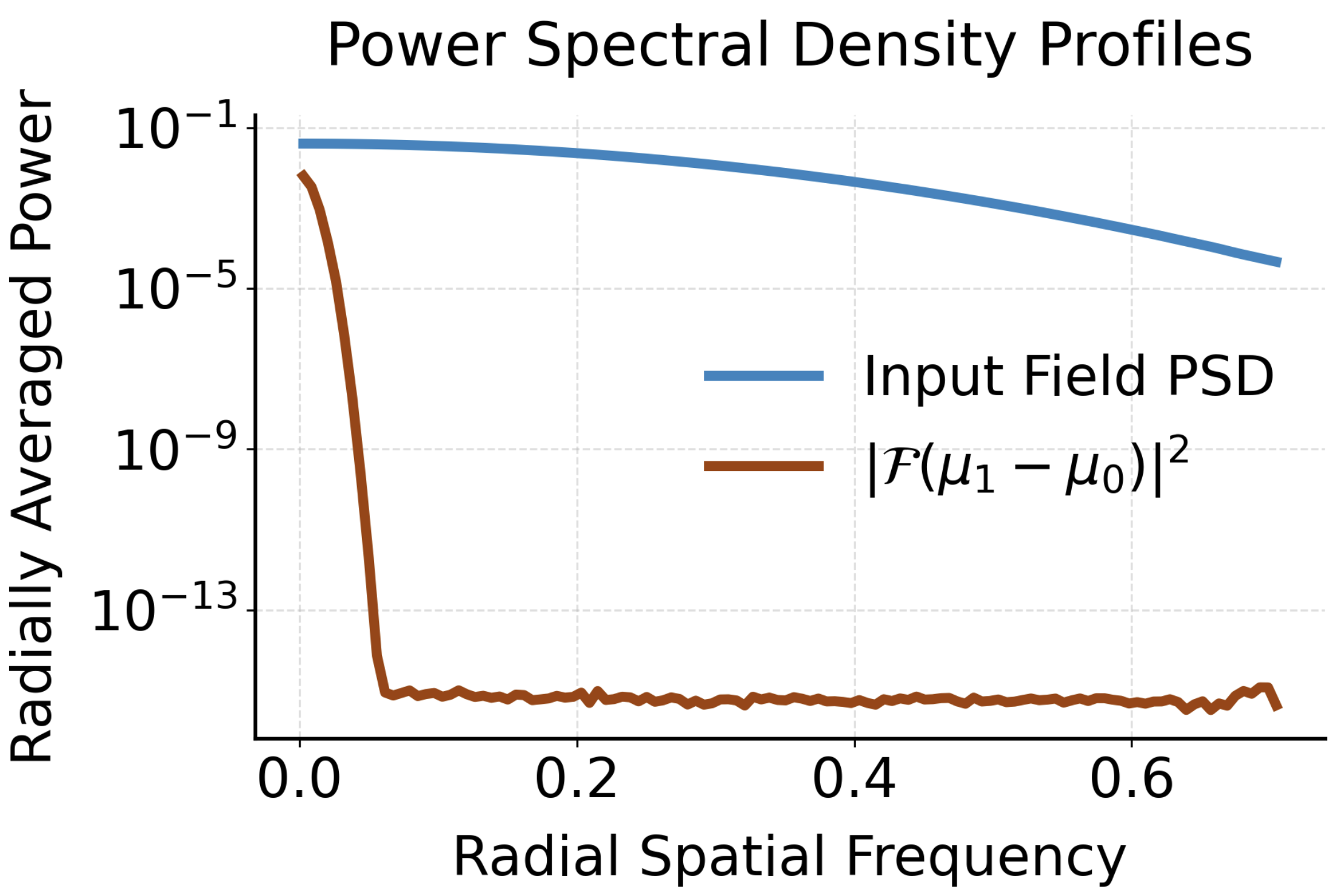}%
\label{fig:rrr}}
\caption{(a) Radially averaged MTF and power-spectrum comparison for a conventional lens and an optimized phase mask. In this setting, the conventional lens' MTF is higher across most spatial frequencies, while the optimized phase mask's MTF suppresses higher frequencies. (b) For the object classes, here are the radial PSD profiles which show the input field and the spectral concentration of $|\mathrm{FFT}(\mu_1-\mu_0)|^2$. The MTF from the optimized phase mask suppresses the higher within-class variation frequencies (while preserving the low class-mean difference frequencies) which ultimately leads to better SNR at the detector.}
\label{fig:radial_mtf}
\end{figure}

For the masked readout sweep, we set the coarse-measurement noise variance to $
\sigma^2(s)=s^2\sigma_{\mathrm{fine}}^2$,
so that it is proportional to the detector area, matching the assumption used in our experiments. The overall 
$s^2$ factor cancels between signal, covariance, and 
detector-noise terms in Eq.~\eqref{eq:sectionV_exact_d2_dft} (Eq.~\ref{eq:d2_s_normalized}). The benefit of 
larger $s$ therefore comes from the shape of the 
detector transfer function---stronger detector-side 
averaging and stronger alias suppression---not from 
collecting more photons per measurement.

Eq.~\eqref{eq:sectionV_exact_d2_dft} clarifies the 
mechanism behind the trends in Figs.~\ref{fig:binary_sweeps}(c) and ~\ref{fig:binary_noise}. In the noiseless case, the denominator contains only aliased within-class power, so $\widehat H$ has maximum 
leverage where detector-side averaging is weakest: at 
small $s$, choosing $\widehat H$ to preserve low-frequency 
class-separating structure while suppressing higher-frequency 
nuisance branches substantially reduces the denominator 
with little loss in the numerator. This is why the gain 
from optimization is largest at small $s$ in Fig.~\ref{fig:binary_sweeps}(c). Once detector noise is 
present, $\widehat H$ reshapes the signal and covariance 
terms but cannot suppress the additive noise floor 
$k^2\sigma_{\mathrm{fine}}^2$. At small $s$ this 
floor dominates and limits how much optics can help; at large $s$ detector-side averaging has already 
suppressed most nuisance variation, leaving little 
headroom. The largest gap between optimized and baseline 
therefore shifts toward intermediate $s$, where neither limit 
dominates, as seen in Fig.~\ref{fig:binary_noise}.

\subsection{Connection to the Modulation Transfer Function}
\label{sec:mtf_connection}
The modulation transfer function (MTF), defined as the magnitude of the OTF $|\hat{H}|$, characterizes how strongly the imaging system transmits each spatial frequency component. The MTF is normalized to 1 at DC and  typically has a low-pass character~\cite{hopkins1955frequency}. As shown in Fig.~\ref{fig:radial_mtf}a, the optimized phase mask preserves low-frequency content relative to a conventional focusing lens while more strongly attenuating high frequencies. In the detector-limited regime, these high-frequency components contribute more to within-class variability and aliasing than to between-class separation---their suppression therefore reduces nuisance variation while retaining the low-frequency structure that carries the bulk of class-mean differences. Whether high-frequency suppression is beneficial depends on the spectral distribution of the discriminative signal; in Fig.~\ref{fig:radial_mtf}(b), that signal is concentrated at low frequencies, making the phase mask's attenuation particularly well-matched to the task. 

The frequency-shifted synthetic controls in Fig.~\ref{fig:mid-freq} further illustrate this point; implementation details are provided in Appendix~\ref{app:freq_shifted_binary}. In these controls, the class-mean difference is shifted toward mid and high spatial frequencies while the within-class covariance, optical model, detector readout, and training procedure are kept fixed. When the class signal is shifted to mid spatial frequencies, optical co-design still improves masked-readout performance, with the largest gain occurring at intermediate readout sizes. When the discriminative signal is shifted to higher spatial frequencies, the advantage of joint optimization decreases as the readout window grows. The drop in both \(d^2\) and test accuracy at larger readout sizes in Fig.~\ref{fig:mid-freq}(b)--(c) occurs because increasing \(s\) makes the masked readout increasingly similar to block-sum averaging. In the mid- and high-frequency controls, the class-mean difference has been shifted away from the lowest spatial frequencies, so this detector-side averaging suppresses not only aliased nuisance variation but also the spatial frequencies that carry the discriminative signal. This is consistent with the MTF interpretation: low-pass optical or detector-side filtering is beneficial when high frequencies mainly carry nuisance variation, but it becomes less favorable when those frequencies contain task-relevant class information.
\begin{figure}[h!]
\centering
\subfloat[]{\includegraphics[width=3.2in]{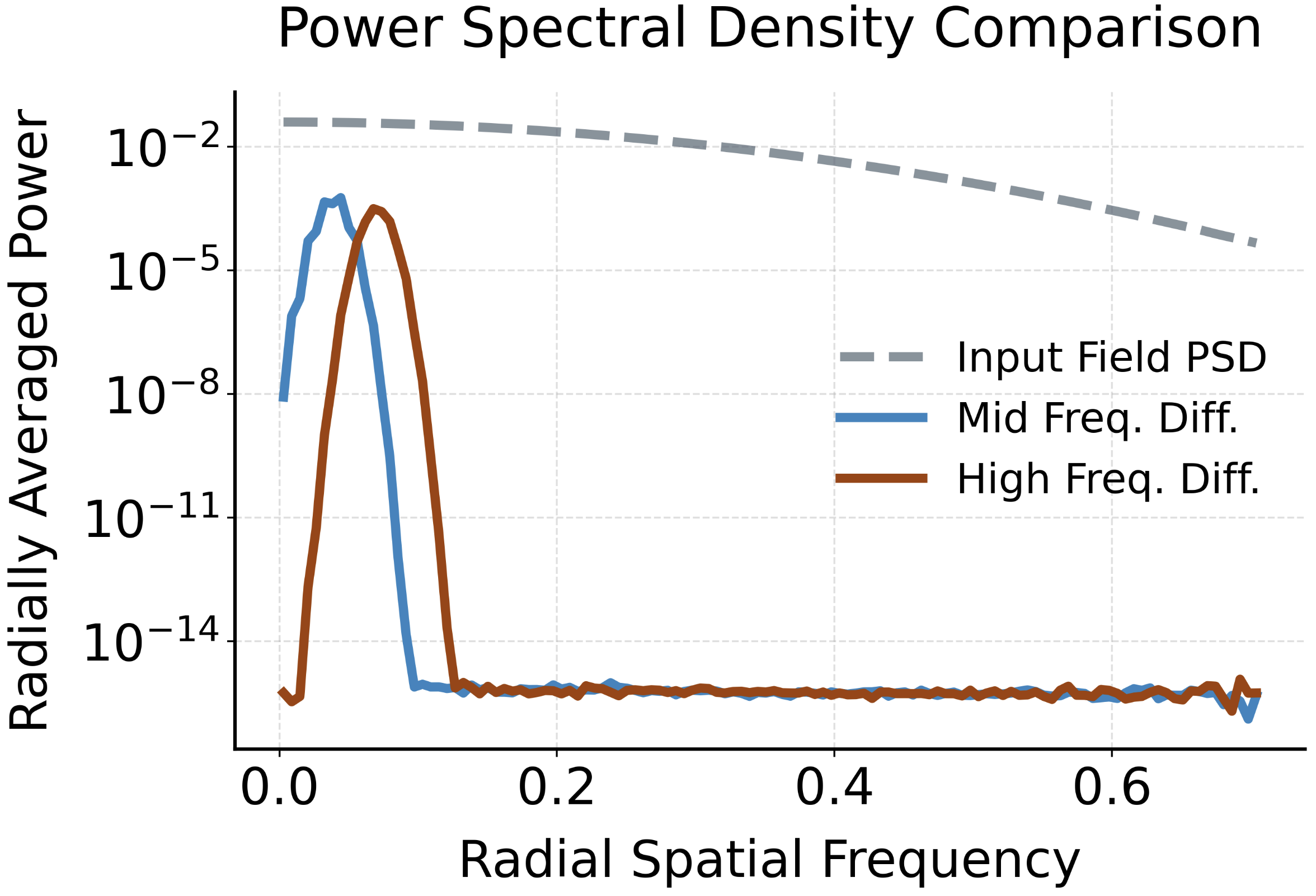}%
\label{fig:psd-mid-high}}
\hfil
\subfloat[]{\includegraphics[width=3.2in]{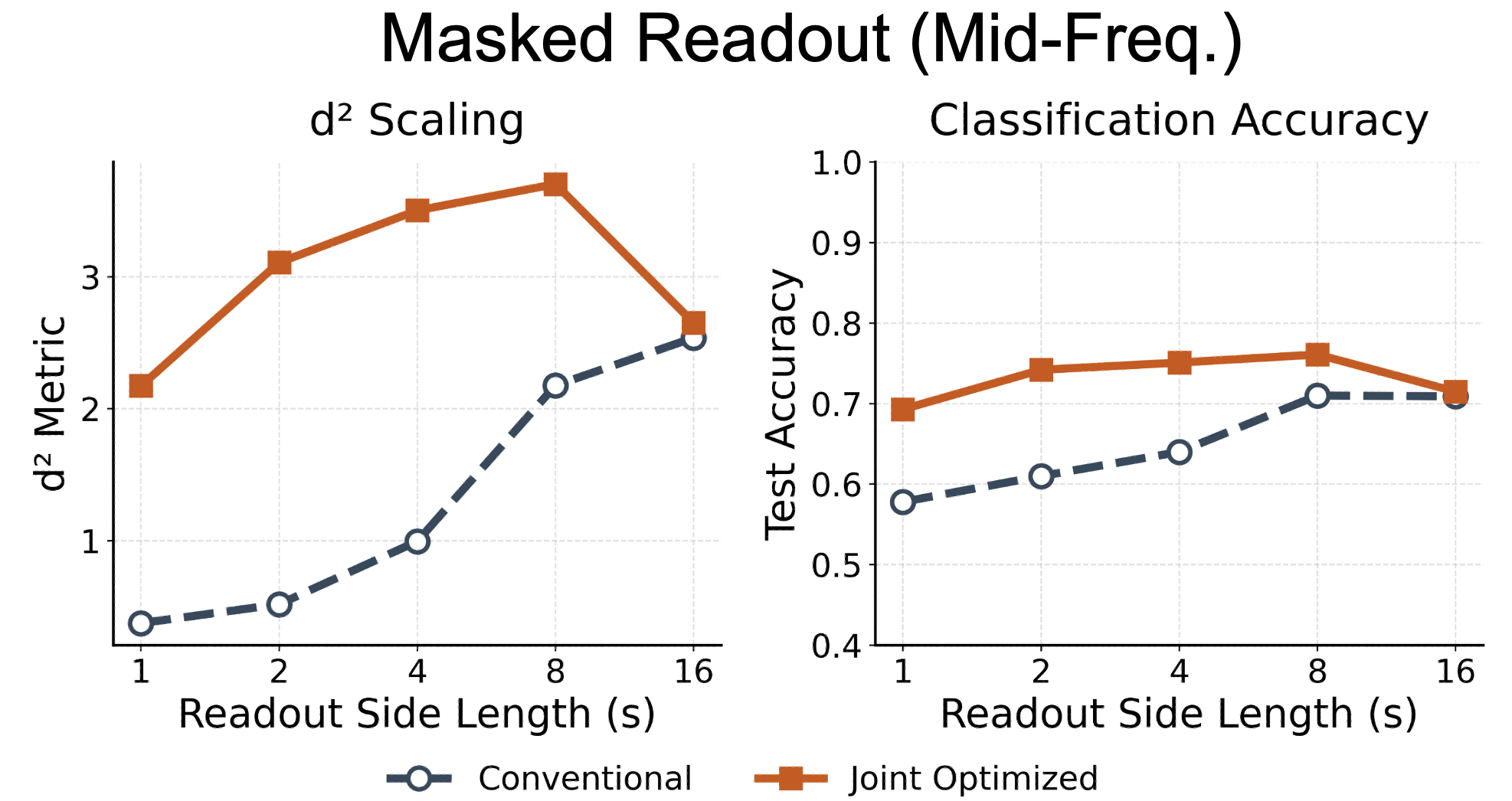}%
\label{fig:mid-accuracy}}
\hfil
\subfloat[]{\includegraphics[width=3.2in]{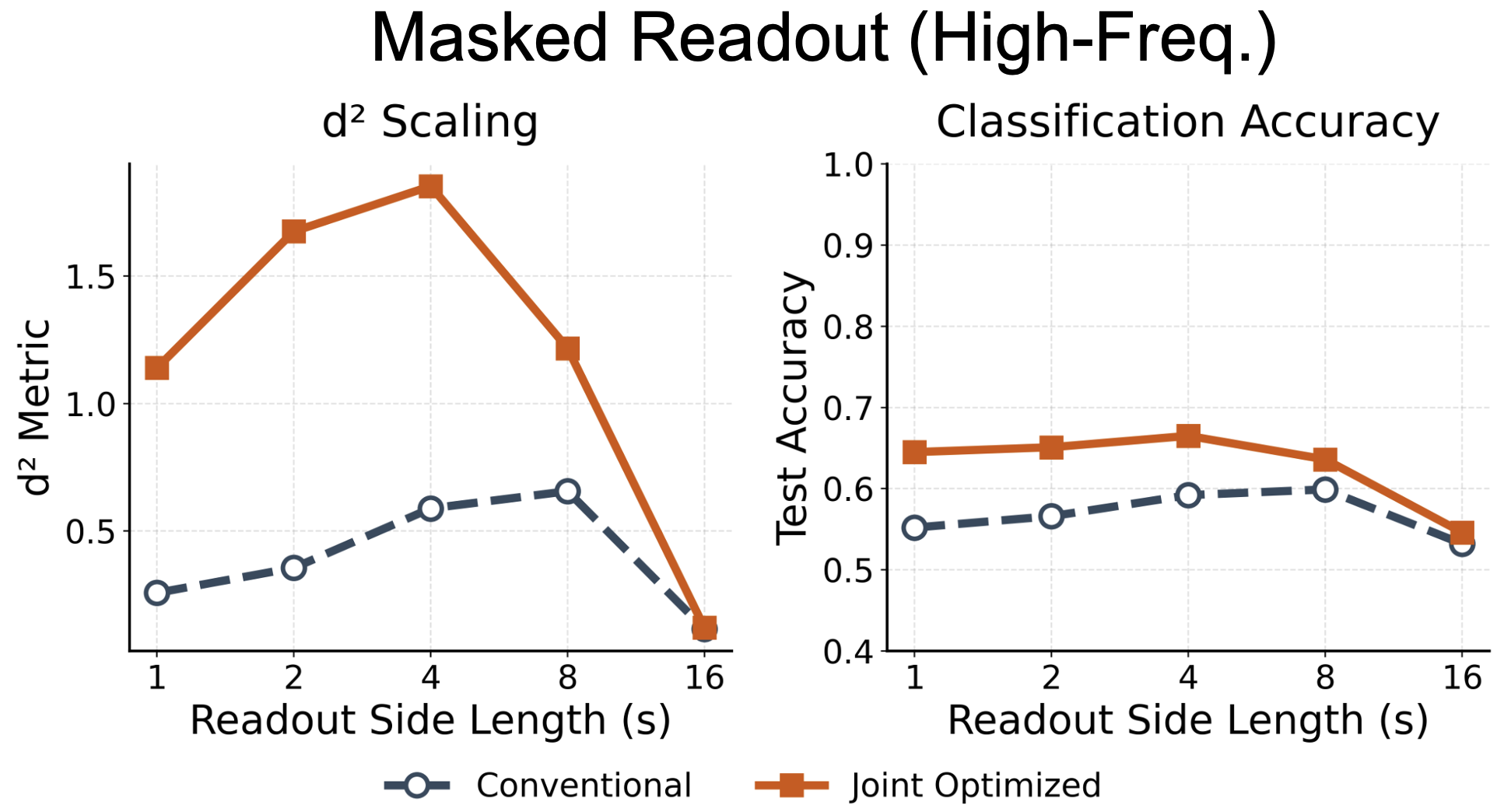}%
\label{fig:high-accuracy}}
\caption{
Frequency-shifted binary synthetic controls under the same masked-readout setting as Fig.~3, with $N=256$, $k=16$, and $2\%$ detector noise. 
(a) Radially averaged power spectra for the input WSS covariance and the class-mean differences in the mid- and high-frequency conditions. The discriminative spectrum is shifted away from the low-frequency regime used in Fig.~3 while the within-class covariance model is kept fixed.
(b) Masked-readout sweep for the mid-frequency condition. Joint optimization improves both the separability proxy $d^2$ and classification accuracy over the conventional lens, with the largest gains occurring at intermediate readout sizes.
(c) Masked-readout sweep for the high-frequency condition. The optimized system improves performance at small and intermediate readout sizes, but both systems degrade as the readout window becomes large and approaches block-sum averaging, because detector-side averaging suppresses the high-frequency discriminative signal itself.
}
\label{fig:mid-freq}
\end{figure}

\section{Additional Experiments}

\begin{figure*}[t]
\centering
\subfloat[]{\includegraphics[width=\textwidth]{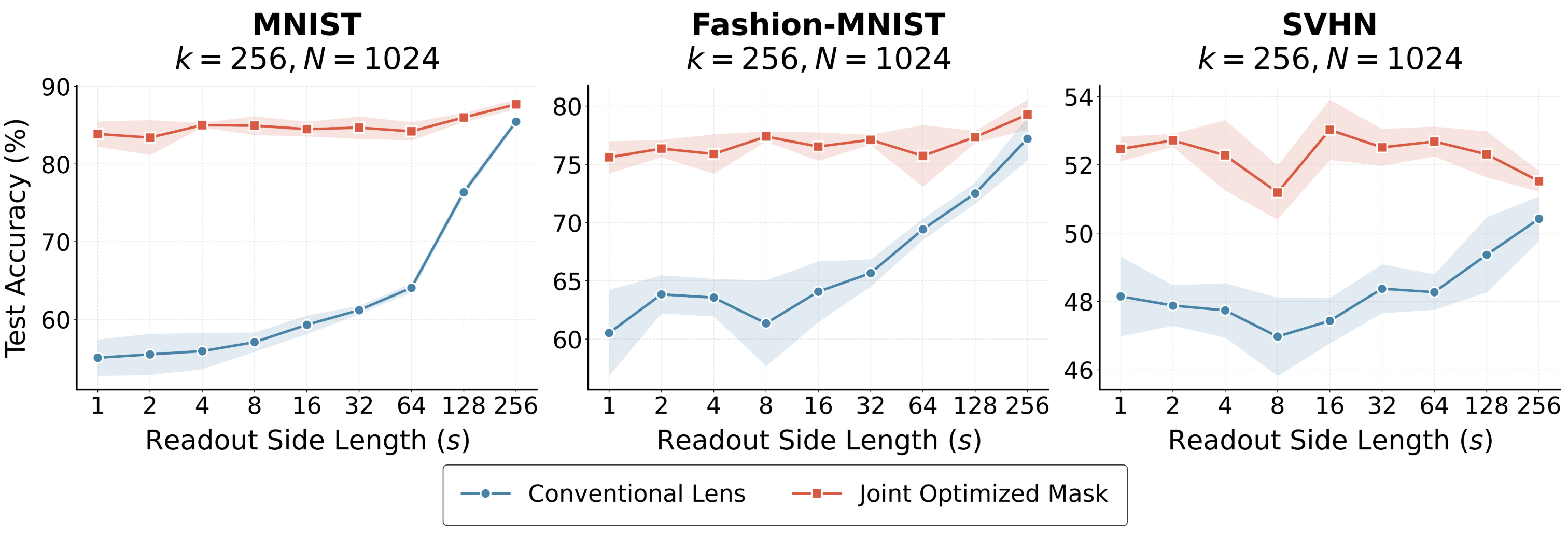}%
\label{fig_first_case}}
\hfil
\subfloat[]{\includegraphics[width=2.3in]{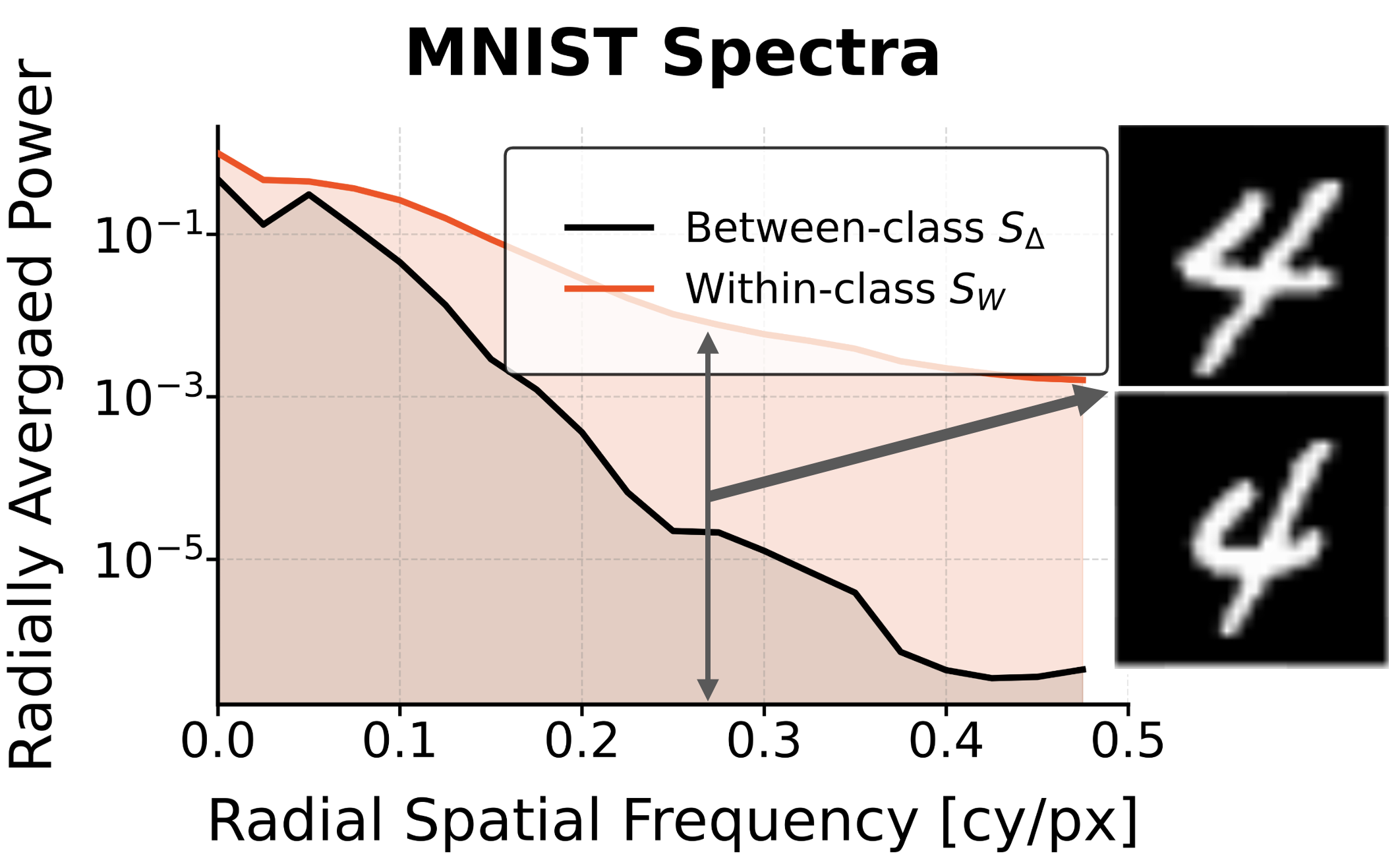}}%
\hfil
\subfloat[]{\includegraphics[width=2.1in]{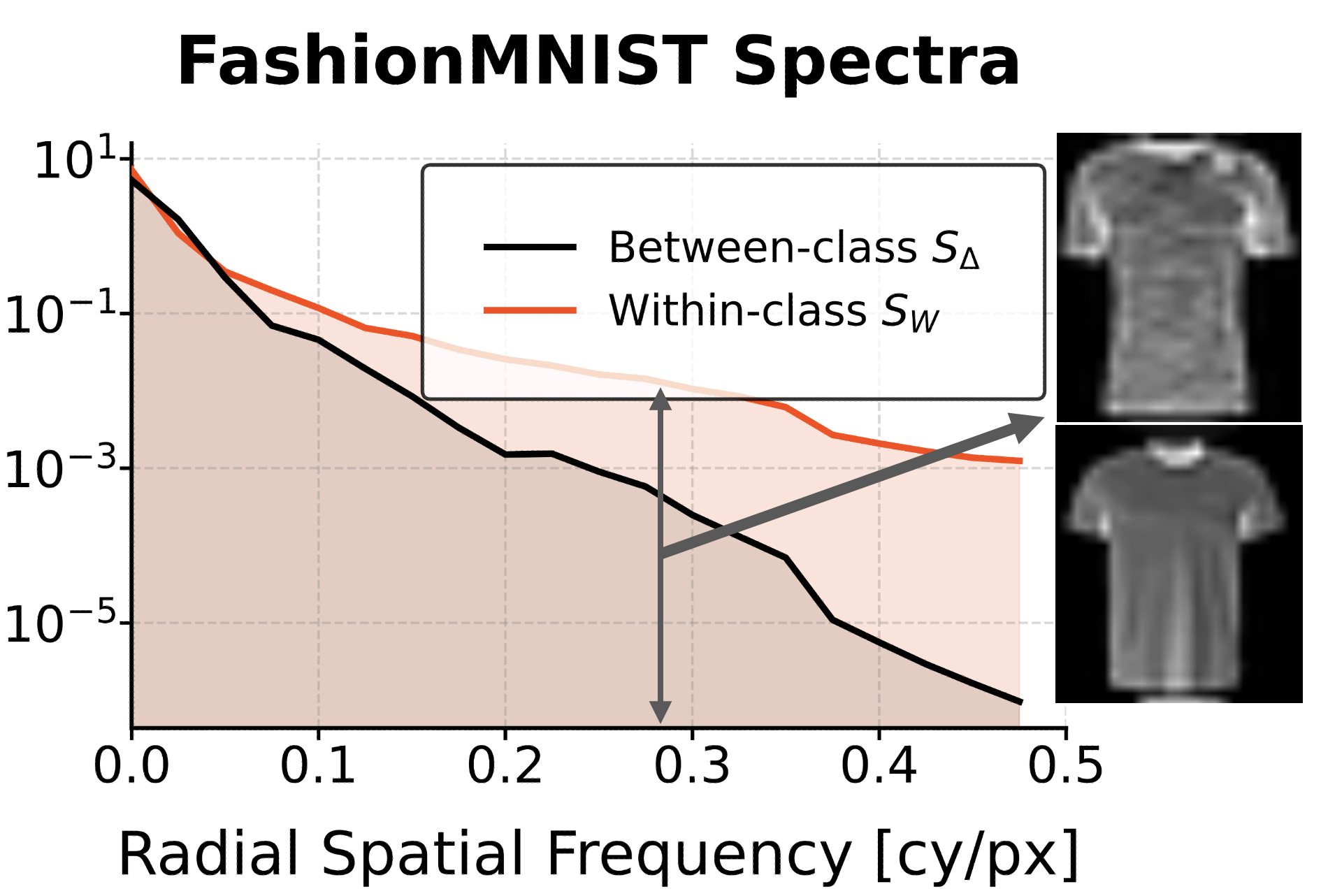}}%
\hfil
\subfloat[]{\includegraphics[width=2.4in]{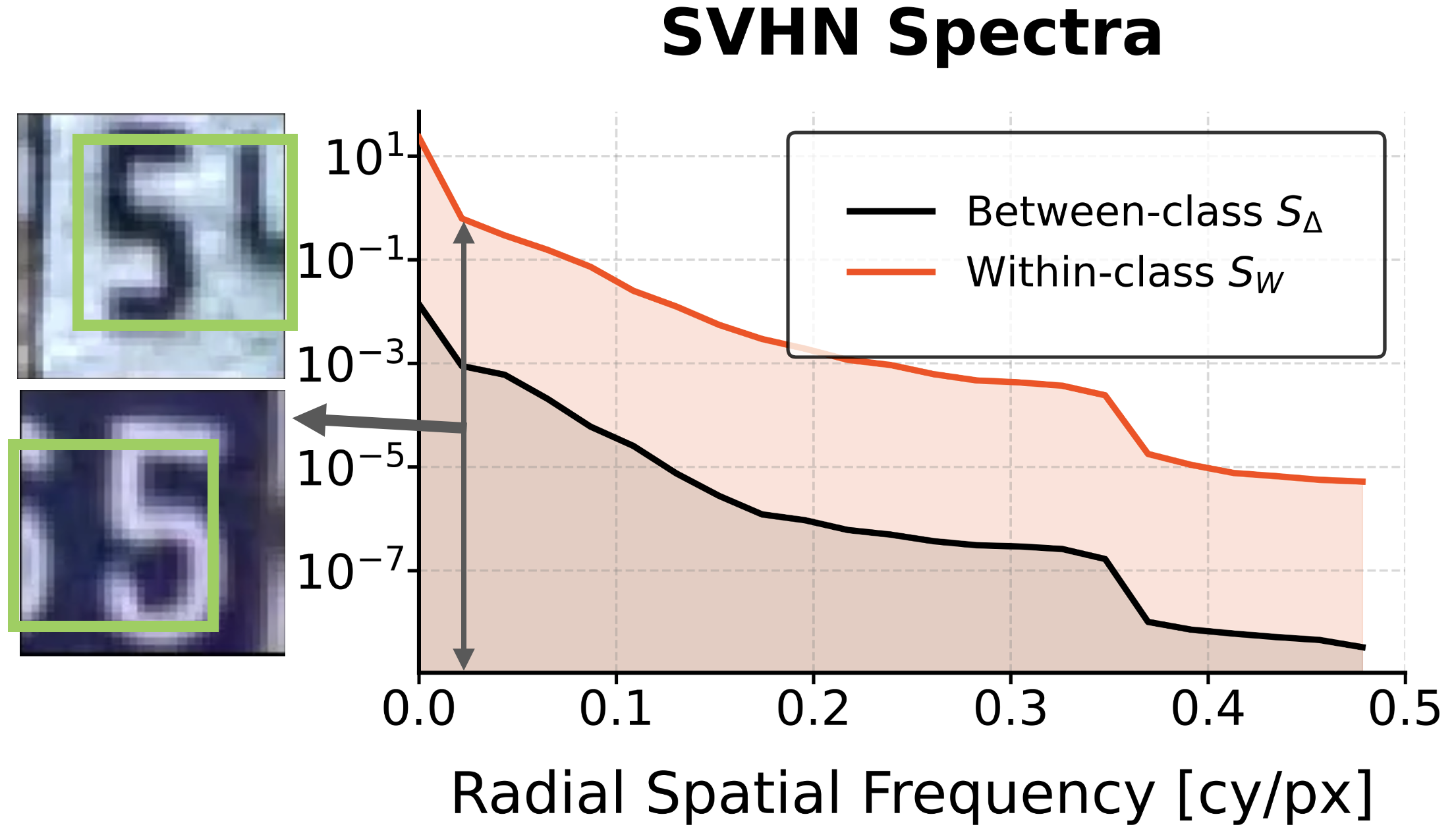}}%
\caption{\textbf{ Masked-readout classification on MNIST, FashionMNIST, and SVHN}. (a) Test accuracy as a function of within-block readout side length $s$ for MNIST, FashionMNIST, and SVHN. For each dataset, the detector is partitioned into non-overlapping $k\times k$ blocks with $N'=4$ ($k=256, N=1024$), and the readout sums a centered $s\times s$ region within each block. A fixed conventional focusing lens with a trained neural-network classifier is compared against an end-to-end optimized system in which the optical phase mask and neural-network classifier are jointly trained. Results show mean test accuracy across five random seeds; shaded regions indicate $\pm 1$ standard deviation. (b)–(d) Azimuthally averaged radial profiles of the between-class spectrum 
$S_{\Delta}(\omega)$ and within-class spectrum 
$S_{W}({\omega})$ for MNIST, FashionMNIST, and SVHN, respectively}
\label{fig:additional_datasets}
\end{figure*}

\subsection{Benchmark Datasets: MNIST, FashionMNIST, and SVHN}
\label{sec:exp_additional_datasets}

We next evaluate the same hypothesis on standard image classification datasets. The goal is not only 
to measure absolute accuracy but also to determine 
whether the regime distinction observed in the synthetic 
setting persists for more complex visual data. Across 
all datasets, we compare conventional lens-based imaging against 
end-to-end optimized passive optical preprocessing under varying 
detector readout constraints, using deep neural networks 
as the classifier to accommodate highly nonlinear 
decision boundaries.  The fine-grid resolution for these simulations is $N=1024$.

Guided by the binary and $K=4$ toy experiments, we set 
detector resolution $N' = 4$ to operate in a strongly 
detector-limited regime where differences between 
conventional and jointly optimized systems are pronounced. 
This choice also yields a large block size $k$, enabling 
a broad sweep over within-block readout size $s$ and a 
clear interpolation between sparse and block-sum readout. At greater detector resolutions, the conventional baseline becomes stronger and the margin for end-to-end improvement correspondingly shrinks, as illustrated by a $N'=8$ masked readout sweep in Fig.~\ref{fig:k128mnist}, where the baseline is already performing relatively well at low noise.

To test whether the masked-readout trends persist on standard vision
benchmarks, we evaluate MNIST, FashionMNIST, and SVHN \cite{Deng2012,lecun1998mnist,xiao2017fashionmnistnovelimagedataset, goodfellow2014svhn} under the same
detector-limited setting. For each dataset and each $s$, we use the cross-entropy loss as our training objective. We report mean test
accuracy across five random seeds in
Fig.~\ref{fig:additional_datasets}a. Implementation details are provided in Appendix~\ref{app:additional_dataset_details}.

\paragraph{Between-class and within-class spectra}
The analysis from Section~\ref{sub:gaussian_setup} to Section~\ref{sec:detector_configs} 
expresses separability in terms of two Fourier-domain quantities: the 
class-mean difference \(\Delta\widehat{\mu}[\ell]\) and the within-class 
spectrum \(\widehat{S}[\ell]\). To connect this to the benchmark datasets, 
we estimate empirical multiclass analogs of these quantities---the 
between-class spectrum \(S_{\Delta}\) and the within-class spectrum 
\(S_W\)---directly from the training data. For each class \(c \in \{1,\dots,K\}\), let
\(\mu_c \in \mathbb{R}^{H \times W}\) denote the empirical class mean image, and let
\(\widehat{\mu}_c(\boldsymbol{\omega})=\mathcal{F}\{\mu_c\}(\boldsymbol{\omega})\)
denote its two-dimensional discrete Fourier transform. We define the between-class spectrum as
\begin{equation*}
S_{\Delta}(\boldsymbol{\omega})
=
\frac{1}{2K^2}
\sum_{i=1}^{K}\sum_{j=1}^{K}
\left|
\widehat{\mu}_i(\boldsymbol{\omega}) -
\widehat{\mu}_j(\boldsymbol{\omega})
\right|^2 .
\end{equation*}
The magnitude-squared term gives the spectral power of each pairwise class-mean difference at frequency \(\boldsymbol{\omega}\); averaging over all class pairs summarizes which spatial frequencies carry class-discriminative structure.

For the within-class spectrum, let \(x \sim p(x \mid C=c)\) denote a sample from
class \(c\), with residual \(r_c = x - \mu_c\) and
\(\widehat{r}_c(\boldsymbol{\omega}) = \mathcal{F}\{r_c\}(\boldsymbol{\omega})\).
The within-class spectrum is
\begin{equation*}
S_W(\boldsymbol{\omega})
=
\sum_{c=1}^{K} \pi_c \,
\mathbb{E}\!\left[
\left|
\widehat{r}_c(\boldsymbol{\omega})
\right|^2
\,\middle|\, C=c
\right],
\end{equation*}
where \(\pi_c = P(C=c)\) is the class prior---the probability that a sample belongs to class \(c\). The prior weights each class by its frequency in the data, so classes that occur more often contribute more to the aggregate within-class spectrum. We use uniform priors \(\pi_c = 1/K\) in our experiments, so each class contributes equally.

To summarize these spectra as a function of spatial-frequency magnitude alone, we compute their azimuthally averaged radial profiles, yielding 1D curves $S_{\Delta}(\omega)$ and $S_W(\omega)$ versus normalized radial frequency $\omega$. We emphasize that the underlying images and spectra are not generally radially symmetric; the azimuthal average is used only as a compact visual summary of how between-class and within-class energy are distributed across spatial frequency scales.

The performance trends in Fig.~\ref{fig:additional_datasets} are also consistent with how the expected MTF of a learned phase mask can suppress higher-frequency nuisance content, so it preferentially preserves low-frequency class-discriminative structure while attenuating higher-frequency within-class variation. For MNIST, between-class differences are concentrated at low frequencies, while much of the within-class energy lies at higher frequencies, so the optimized optics can suppress nuisance variation with little loss of discriminative content. For FashionMNIST, the jointly optimized system likewise improves over the conventional lens baseline. A plausible explanation, consistent with the spectral profiles, is that although within-class variability remains substantial even in the low-frequency regime where between-class differences are strongest, suppressing higher-frequency within-class content is still beneficial under detector-limited readout. For SVHN, the within-class spectrum remains high in the same low-frequency band where between-class differences are concentrated, leaving little room for low-pass optical shaping to suppress nuisance variation without also attenuating discriminative content; the framework therefore predicts a smaller gain from optimization, consistent with the observed result. While the separability measure $d^2$ was derived in the synthetic Gaussian setting, the qualitative trends it predicts are reflected empirically by deep classifiers on these benchmark datasets, suggesting that the underlying spectral mechanism extends beyond the strict Gaussian model.

\subsection{Effect of Detector Noise: MNIST}
As in Section~\ref{sub:binary_noise}, we study the effect of detector noise on the relative performance of end-to-end optimization and a focusing lens. Fig.~\ref{fig:mnist_noise} shows the same overall tendency as in the toy binary experiment: increasing detector noise reduces the gain from joint optimization, especially at small readout side lengths \(s\). As detector noise increases, the absolute accuracy of both the conventional and jointly optimized systems decreases, and the performance gap between them also shrinks, especially at small \(s\), because reduced measurement SNR limits the separability improvement produced by optical co-design. For each noise level, we scaled the noise at each fine-grid pixel such that the noise standard deviation was \{0.5\%, 1\%, 2\%, 5\%\} of the mean object intensity per fine-grid pixel, using the same per-unit-area noise convention as in Section~\ref{sub:binary_noise}. For all the noise levels plotted, the gap is largest at intermediate \(s\), yielding the same qualitative behavior as in Fig.~\ref{fig:binary_noise}.

\begin{figure}
    \centering
    \includegraphics[width=0.9\linewidth]{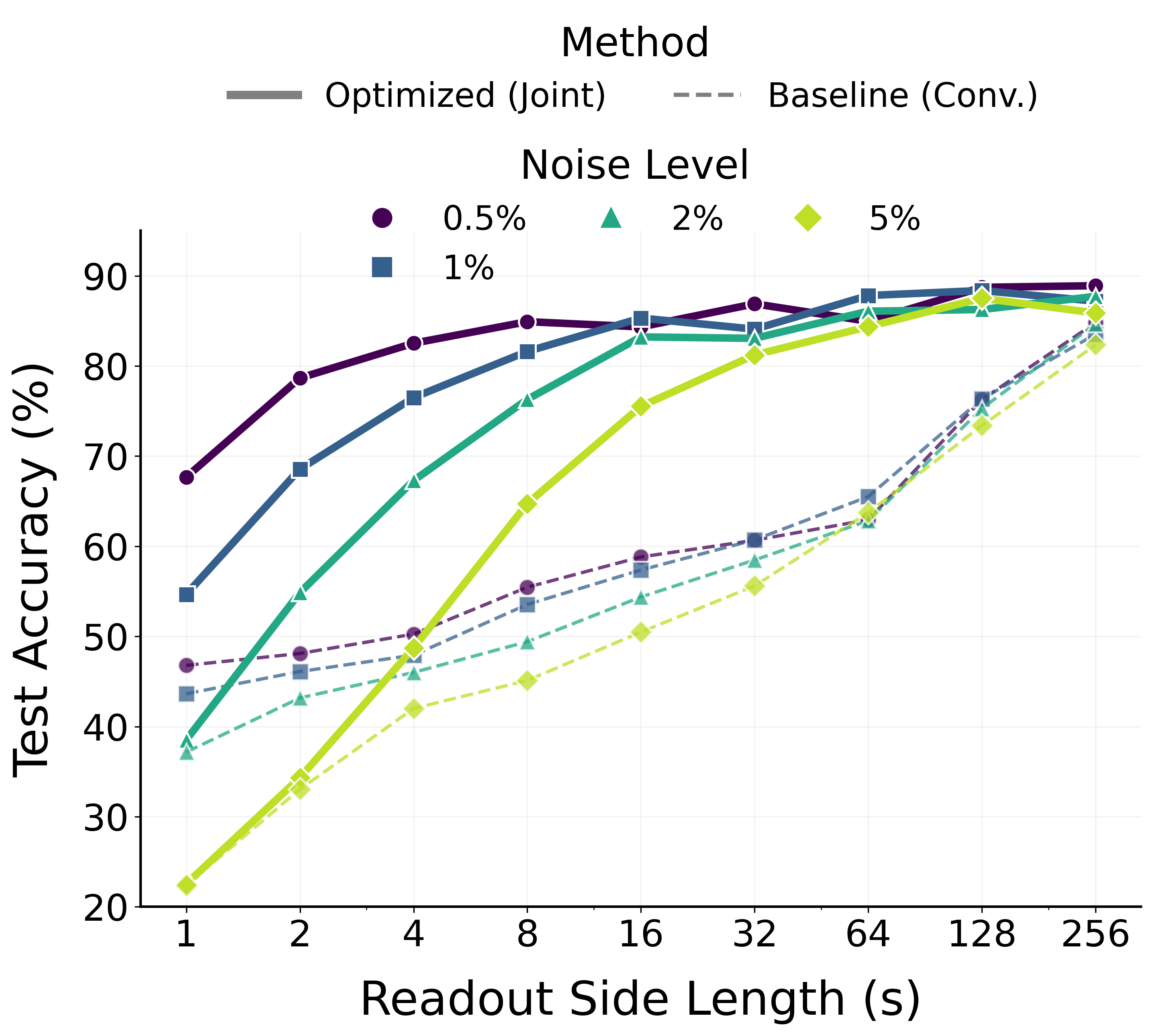}
    \caption{Accuracy versus within-block readout side length $s$ for conventional and jointly optimized systems across detector-noise levels. Dashed lines denote joint optimization and solid lines denote the conventional lens; marker/color indicates noise level. Increasing detector noise reduces the accuracy advantage of joint optimization, especially at small $s$.}
    \label{fig:mnist_noise}
\end{figure}

\section{Discussion and Conclusion}

Our results offer a framework for understanding when end-to-end optimization of optical--computational systems is beneficial for object classification, with particular emphasis on detector-side constraints. We prove that under full detector readout (Section~\ref{sub:mi_proof}), no passive incoherent phase mask exceeds the ideal-channel mutual information between detector measurements and class labels. We approximate this ceiling in our simulations with a conventional focusing lens and verify empirically that joint optimization yields no gain in this regime. Under detector-limited readout, the absolute performance of the conventional baseline can degrade more sharply, whereas the jointly optimized system can preserve more class-discriminative information through task-aware measurement formation. Consequently, the performance gap between the two systems is largest in detector-limited, high-SNR regimes and decreases as detector readout becomes less constrained or detector noise increases. The mechanism we identify also depends on the spectral 
structure of the task: co-design provides the largest 
gains when class-discriminative content sits at lower 
spatial frequencies than within-class variation, a 
condition we verify on the benchmark datasets and probe 
directly with frequency-shifted controls. For more complex natural-image datasets, we expect the gains from this pipeline to be smaller when classification depends on fine edges, textures, or other high-frequency details. In that setting, high-frequency content may carry task-relevant information rather than primarily nuisance variation, reducing the benefit of the low-pass filtering behavior observed in our optimized phase masks. Richer optical front ends, such as non-local metasurfaces, may provide more flexible ways to mix and route image information onto the detector~\cite{Shastri2022}.

More broadly, the present analysis isolates detector-side degrees of freedom within a passive, monochromatic, shift-invariant, incoherent imaging model. In more general optical systems, additional task-relevant degrees of freedom may be available through wavelength, polarization, temporal modulation, or partial coherence. Extending the present framework to incorporate these modalities is a promising direction for future work. Overall, these results suggest that the main value of learned passive optics in the incoherent setting lies not in improving unconstrained imaging, but in enabling task-aware measurement formation under sensing constraints. This principle provides a foundation for designing future 
optical--computational classification systems in which co-design 
reduces the sensing burden required for task performance. In settings 
where dense detector readout, digitization, or data transfer impose 
latency, bandwidth, or energy constraints, learned optical front ends 
can enable accurate inference from fewer or coarser measurements---
making optical co-design most valuable when the cost of acquiring and 
moving pixels dominates the cost of computation.

\section*{Acknowledgments}
The authors gratefully acknowledge insightful discussions with Zhuo Chen. Archer Wang and Joshua Chen are supported by the NSF
graduate research fellowship under Grant Number 2141064. This material is based upon work sponsored by the U.S. Army DEVCOM ARL Army Research Office through the MIT Institute for Soldier Nanotechnologies under Cooperative Agreement number W911NF-23-2-0121. Research was sponsored by the Department of the Air Force Artificial Intelligence Accelerator and was accomplished under Cooperative Agreement Number FA8750-19-2-1000. The views and conclusions contained in this document are those of the authors and should not be interpreted as representing the official policies, either expressed or implied, of the Department of the Air Force or the U.S. Government. The U.S. Government is authorized to reproduce and distribute reprints for Government purposes notwithstanding any copyright notation herein. The authors used OpenAI's ChatGPT, Anthropic's Claude, and Google's Gemini to assist with manuscript editing,
wording suggestions, and preparation of selected plotting
scripts. The authors reviewed and edited all AI-assisted content and take full
responsibility for the final manuscript.

{
\appendices

\section{Proof of Theorem~\ref{thm:mi}}
\label{appendix:thms}

\begin{proof}
\label{app:theorem1-proof}

Since $h[n] \ge 0$ and $\sum_n h[n] \le 1$, Young's 
convolution inequality~\cite{bogachev2007measure1} gives 
$\|A\|_2 \le 1$, equivalently 
$I - AA^\top \succeq 0$~\cite{horn2012matrix}. We may 
therefore decompose the noise as
\[
\epsilon \stackrel{d}{=} A\eta + \zeta, 
\quad \eta \sim \mathcal{N}(0, \sigma^2 I), 
\quad \zeta \sim \mathcal{N}\bigl(0,\, 
\sigma^2 (I - AA^\top)\bigr),
\]
with $\eta \perp \zeta$. Then 
$AX + \epsilon = A(X + \eta) + \zeta$, yielding the 
Markov chain
\[
C \;\to\; X \;\to\; X + \eta \;\to\; A(X + \eta) 
\;\to\; A(X + \eta) + \zeta = Y.
\]
By the data-processing inequality,
$
I(C; Y) \;\le\; I(C; X + \eta) 
\;=\; I(C; X + \epsilon)$.
\end{proof}

\section{Separability Under Detector-Limited Readouts}
\label{app:exact_alias}

Here, we derive the exact finite-dimensional separability
expressions for detector downsampling using the discrete Fourier transform, without assuming that the discriminative spectrum is confined to
the baseband. We specify the detector readout 
operator referenced in the main text, then derive 
the per-bin separability expression and its endpoint 
cases.

\subsection{Detector Readout Operator}
\label{app:detector-operator}

The detector readout operator $D_W$ introduced in 
Section~\ref{sec:detector_binning} maps the fine-grid 
intensity $Y_{\text{fine}} \in \mathbb{R}^{N \times N}$ 
to a coarse-grid measurement array of size 
$N' \times N'$, where $N' = N/k$. Concretely, for 
$(p,q) \in \{0, \ldots, N'-1\}^2$,
\begin{equation}
(D_W Y_{\text{fine}})[p,q] := 
\sum_{i=0}^{k-1} \sum_{j=0}^{k-1} W[i,j]\,
Y_{\text{fine}}[pk+i, qk+j],
\label{eq:DW_definition}
\end{equation}
so all $k^2$ fine-grid pixels within each block 
contribute to the measurement under the weighting 
mask $W \in \mathbb{R}^{k \times k}$. The three readout 
schemes used in the main text correspond to different 
choices of $W$: $W \equiv 1$ yields block-sum readout; 
a one-hot $W$ yields sparse readout; a centered 
$s \times s$ indicator ($s \le k$) yields masked 
readout. The final measurement is 
$Y = D_W Y_{\text{fine}} + \epsilon$ with 
$\epsilon[p,q] \stackrel{\text{i.i.d.}}{\sim} 
\mathcal{N}(0, \sigma^2)$.

\subsection{Channel Model Per Frequency Bin}
\label{app:exact_alias_channel}

Consider a 1D signal of length \(N\), with downsampling factor \(k\) such that
\(k \mid N\). Let \(\widehat{X}[\ell]\) denote the fine-grid DFT coefficients,
for \(\ell=0,1,\dots,N-1\), and let \(\widehat{Y}[m]\) denote the DFT
coefficients on the coarse grid of length \(N/k\), for
\(m=0,1,\dots,N/k-1\).

For each coarse-grid frequency bin \(m\), downsampling by a factor \(k\)
aliases together the \(k\) fine-grid bins
\[
\ell_r := m + r\frac{N}{k}, \qquad r=0,1,\dots,k-1.
\]

Let \(\widehat{H}[\ell]\) denote the incoherent OTF sampled on the fine-grid
DFT frequencies, and let \(\widehat{q}_r[m]\) denote the detector-dependent
weight applied to the aliased branch \(\ell_r\). Define the combined
detector--optics factor $
g_r[m] := \widehat{q}_r[m]\widehat{H}[\ell_r]$.

Then the coarse-grid scalar channel at bin \(m\) is
\begin{equation}
\widehat{Y}[m]
=
\frac{1}{k}\sum_{r=0}^{k-1} g_r[m]\widehat{X}[\ell_r] + \widehat{N}[m],
\label{eq:exact_alias_channel_dft}
\end{equation}
where \(\widehat{N}[m]\sim \mathcal{N}_{\mathbb C}(0,\sigma^2)\) is additive
measurement noise.

Assume that, conditioned on class \(C=c\), the fine-grid DFT coefficients are
complex Gaussian: $
\widehat{X}[\ell]\mid C=c \sim
\mathcal{N}_{\mathbb C}\!\big(\widehat{\mu}_c[\ell],\,\widehat{S}[\ell]\big)$,
with shared within-class spectrum \(\widehat{S}[\ell]\) across classes. Define
the Fourier-domain mean difference $
\Delta\widehat{\mu}[\ell] := \widehat{\mu}_1[\ell]-\widehat{\mu}_0[\ell]$.

Under the diagonal Fourier-domain covariance assumption, the aliased
coefficients \(\{\widehat{X}[\ell_r]\}_{r=0}^{k-1}\) are uncorrelated within
each class. Hence, from \eqref{eq:exact_alias_channel_dft}, the class-conditional
mean difference at coarse-grid bin \(m\) is
\[
\Delta\widehat{\mu}_Y[m]
=
\frac{1}{k}\sum_{r=0}^{k-1} g_r[m]\Delta\widehat{\mu}[\ell_r],
\]
and the corresponding within-class variance is
\[
\operatorname{Var}(\widehat{Y}[m])
=
\frac{1}{k^2}\sum_{r=0}^{k-1}|g_r[m]|^2 \widehat{S}[\ell_r] + \sigma^2.
\]

Therefore, the exact per-bin contribution to the squared Mahalanobis distance is
\begin{equation}
\widehat{d}^2[m]
=
\frac{\left|\sum_{r=0}^{k-1} g_r[m]\Delta\widehat{\mu}[\ell_r]\right|^2}
{\sum_{r=0}^{k-1}|g_r[m]|^2 \widehat{S}[\ell_r] + k^2\sigma^2 }.
\label{eq:exact_d2_general_dft}
\end{equation}

Moreover, because the Fourier-domain covariance is diagonal and the alias sets
associated with distinct coarse-grid bins \(m\) are disjoint, the coarse-grid
measurements \(\{\widehat{Y}[m]\}_{m=0}^{N/k-1}\) are uncorrelated. Since
\(\widehat{Y}\) is jointly Gaussian, these components are therefore independent,
so the covariance matrix \(\widehat{\Sigma}_Y\) is diagonal. In the
complex-valued setting, the squared Mahalanobis distance is $
d^2
=
(\widehat{m}_1-\widehat{m}_0)^*
\widehat{\Sigma}_Y^{-1}
(\widehat{m}_1-\widehat{m}_0)$,
where \({}^*\) denotes conjugate transpose. Since \(\widehat{\Sigma}_Y\) is
diagonal, this decomposes across coarse-grid frequency bins as
\[
d^2
=
\sum_{m=0}^{N/k-1}
\frac{|\Delta \widehat{m}[m]|^2}{\operatorname{Var}(\widehat{Y}[m])}
=
\sum_{m=0}^{N/k-1} \widehat{d}^2[m].
\]

\subsection{Explicit Expressions For Detector-Limited Readout}
\label{app:exact_endpoints}

The general expression in Eq.~\eqref{eq:exact_d2_general_dft} specializes to
masked readout and to its two endpoints: block-sum and sparse readout.

\subsubsection{Masked Readout}
For an \(s\)-pixel contiguous readout window, ignoring the unit-modulus phase
factor set by the window position, the frequency response on aliased branch
\(\ell_r\) is
\[
\widehat q_{r,s}[m]
=
\sum_{t=0}^{s-1} e^{-j2\pi \ell_r t/N}
=
e^{-j\pi \ell_r(s-1)/N}
\frac{\sin(\pi s\ell_r/N)}{\sin(\pi \ell_r/N)} ,
\]
with the ratio defined by continuity at \(\ell_r=0\). As \(s\) increases, this
response becomes more concentrated near low spatial frequencies.

Writing \(\widehat q_{r,s}[m]=s\,\widetilde q_{r,s}[m]\), and using the
per-unit-area noise convention \(\sigma_s^2=s^2\sigma_{\mathrm{fine}}^2\),
Eq.~\eqref{eq:exact_d2_general_dft} becomes
\begin{equation}
\widehat d_s^2[m]
=
\frac{\left|\sum_{r=0}^{k-1} \widetilde q_{r,s}[m]\,
\widehat H[\ell_r]\Delta\widehat\mu[\ell_r]\right|^2}
{\sum_{r=0}^{k-1} |\widetilde q_{r,s}[m]|^2
|\widehat H[\ell_r]|^2\widehat S[\ell_r]
+
k^2\sigma_{\mathrm{fine}}^2}.
\label{eq:d2_s_normalized}
\end{equation}

\subsubsection{Block-Sum Readout}
\label{app:exact_sum_pool}

Block-sum readout is the endpoint \(s=k\). Its frequency response is
\[
\widehat{B}_k[\ell]
=
\sum_{t=0}^{k-1} e^{-j 2\pi \ell t / N}
=
e^{-j\pi \ell (k-1)/N}
\frac{\sin(\pi k \ell / N)}{\sin(\pi \ell / N)} .
\]
Thus \(\widehat q_r[m]=\widehat B_k[\ell_r]\), giving
\begin{equation*}
d^2_{\mathrm{block}}[m;H]
=
\frac{\left|\sum_{r=0}^{k-1} \widehat{B}_k[\ell_r]\widehat{H}[\ell_r]\Delta\widehat{\mu}[\ell_r]\right|^2}
{\sum_{r=0}^{k-1}|\widehat{B}_k[\ell_r]|^2 |\widehat{H}[\ell_r]|^2 \widehat{S}[\ell_r] + k^2\sigma^2 }.
\end{equation*}

\subsubsection{Sparse Readout}
\label{app:exact_sparse}

Sparse readout is the endpoint \(s=1\). Choosing the measured pixel as the
zero-offset location sets \(\widehat q_r[m]=1\), and other offsets only add
unit-modulus phase factors. Therefore,
\begin{equation*}
d^2_{\mathrm{sparse}}[m;H]
=
\frac{\left|\sum_{r=0}^{k-1} \widehat{H}[\ell_r]\Delta\widehat{\mu}[\ell_r]\right|^2}
{\sum_{r=0}^{k-1}|\widehat{H}[\ell_r]|^2 \widehat{S}[\ell_r] + k^2\sigma^2 }.
\end{equation*}

\section{Experimental Details}
\subsection{Physical imaging model}
The forward imaging model is implemented via angular-spectrum propagation\cite{Matsushima2009BLAS}, and the wavelength is set to $550\mathrm{nm}$. 
The imaging region spans $1\,\mathrm{mm} \times 1\,\mathrm{mm}$. The pupil function is expressed as $P(x,y)=A(x,y)e^{i\phi_\theta(x,y)}$, where $A(x,y)$ is a binary circular aperture of diameter $0.5\mathrm{mm}$, and $\phi_\theta(x,y)$ is a phase profile parameterized by $j_{\max}=30$ Zernike coefficients. This captures the low-order spectral reshaping predicted by our analysis while remaining physically realizable. The object-to-lens and lens-to-detector distances are both
$5\,\mathrm{mm}$, corresponding to $1\times$ magnification.

\subsection{Toy Binary Experiment Details}
\label{app:toy_binary_details}

Here, we set the detector resolution to $N=256$
and use synthetic datasets with
$n_{\mathrm{train}}=20000$ training samples and
$n_{\mathrm{val}}=4000$ validation samples. Each object image
$x\in\mathbb{R}^{N\times N}$ is drawn from one of two equiprobable
classes with shared covariance,
$x \mid C=c \sim \mathcal N(\mu_c,\Sigma)$,
where $\mu_0$ and $\mu_1$ differ by a difference-of-Gaussians
pattern defined on a normalized $[-1,1]\times[-1,1]$ grid. The pattern is constructed from two spatial Gaussian components,
\[
g_{\pm}(x,y)
=
\exp\!\left(
-\frac{(x\mp 0.35)^2+y^2}{2\sigma^2}
\right),
\qquad
\sigma=0.12,
\]
with \(e=g_{-}-g_{+}\). The two class means are then defined as
\[
\mu_0 = b + \frac{\alpha}{2}e,
\qquad
\mu_1 = b - \frac{\alpha}{2}e,
\]
where \(b=0.6\) and \(\alpha=0.01065\). Thus \(\alpha\) is the scale factor
for the inter-class difference pattern \(\mu_0-\mu_1=\alpha e\), while each
class mean is displaced from the base intensity by \((\alpha/2)e\).


We model $\Sigma$ using a WSS approximation, represented by a radially symmetric Gaussian
power spectral density
\[
S_{\mathrm{in}}(\omega_x,\omega_y)
=
\sigma_x^2 \exp\!\left(
-\frac{\omega_x^2 + \omega_y^2}{2\omega_0^2}
\right),
\]
with $\sigma_x=0.20$ and $\omega_0=1.2$. Additive WSS covariance is
introduced at the object plane and propagated through the
imaging system. After incoherent optical filtering with
point-spread function $h$ and transfer function $H$, the
full-resolution measurement PSD is $
S_0(\omega_x,\omega_y)
=
|H(\omega_x,\omega_y)|^2
S_{\mathrm{in}}(\omega_x,\omega_y)$.

Under the WSS and circulant approximation, the covariance diagonalizes in the Fourier domain, allowing efficient evaluation of $d^2$.
\begin{equation}
d^2 =
\sum_{\omega_x,\omega_y}
\frac{|\Delta M(\omega_x,\omega_y)|^2}
{S_0(\omega_x,\omega_y) + \lambda},
\label{eq:d2_expression}
\end{equation}
where $\Delta M$ is the Fourier transform of $\Delta m$, and
$\lambda$ is a small ridge term added for numerical stability.

For detector aggregation, we model the sensor readout as
weighted block pooling within $k\times k$ regions followed by
decimation. The corresponding detector response is incorporated
in the frequency domain, and the resulting noise PSD on the
coarse grid is obtained by applying the pooling filter and
performing 2D alias-folding into the baseband. $d^2$ is then computed on this coarse grid using Eq.~\ref{eq:d2_expression}.

For the optical front end, we use a phase-only lens model whose
phase profile is parameterized by Zernike coefficients with
$j_{\max}=30$. The resulting PSF is used both in the forward
simulation and in the computation of the separability metric.

We study two detector-limited regimes.

\paragraph*{1) Block-sum detector resolution sweep.}
We vary the detector resolution
$N'\in\{2,4,8,16,32\}$,
with block size $k=N/N'$, using uniform
$k\times k$ sum-binning and decimation. For each detector
resolution, we compute the corresponding $d^2$ using the 2D
box response and alias-folding model. To validate the predicted
trend, we also train a linear classifier on simulated detector
measurements. Specifically, the
$N'\times N'$ outputs are flattened,
normalized with BatchNorm (using running statistics), and passed
to a logistic regressor trained with Adam (learning rate
$10^{-2}$), $\ell_2$ penalty $10^{-4}$, and early stopping on
validation accuracy for up to 30 epochs.

\paragraph*{2) Masked-readout sweep.}
We fix $N'=16$, so that $k=16$, and vary the
detector readout budget by selecting a centered $s\times s$
region within each $k\times k$ block and summing the responses
of those pixels. We sweep $
s\in\{1,2,4,8,16\}$.
Thus $s=1$ corresponds to sparse readout and $s=16$ recovers
full sum-pooling within each block.

In both regimes we compare a fixed conventional lens against
end-to-end joint optimization of the optical phase together
with the BatchNorm-plus-logistic-regression classifier. Joint
optimization is performed with Adam (learning rate $10^{-3}$)
for 20 epochs with batch size 128. For each setting we report
both the theoretical separability proxy $d^2$ and the validation
classification accuracy.

\subsection{Frequency-Shifted Binary Synthetic Controls}
\label{app:freq_shifted_binary}

We also ran a frequency-shifted variant of the binary synthetic experiment to test whether the benefit of optical co-design depends on where the class-discriminative signal lies in spatial frequency. This experiment uses the same optical model, WSS covariance, detector readout, noise model, classifier, and training procedure as Appendix~C-B. The only change is the class-mean difference.

Let
\begin{equation*}
e(x,y)
=
\exp\!\left(-\frac{(x+x_0)^2+y^2}{2\sigma_b^2}\right)
-
\exp\!\left(-\frac{(x-x_0)^2+y^2}{2\sigma_b^2}\right),
\end{equation*}
with $x_0=0.35$ and $\sigma_b=0.12$, denote the original two-lobed pattern. The low-frequency condition uses $e(x,y)$ directly. For the mid- and high-frequency conditions, we multiply this same envelope by a sinusoidal carrier,
\begin{equation*}
e_f(x,y) = e(x,y)\cos(2\pi f x).
\end{equation*}
The class means are then
\begin{equation*}
\mu_0^{(f)} = b + \frac{\alpha}{2}\tilde e_f,
\qquad
\mu_1^{(f)} = b - \frac{\alpha}{2}\tilde e_f,
\end{equation*}
where $b=0.6$ and $\alpha=0.01065$. For $f=0$, this exactly recovers the original low-frequency binary experiment. For $f>0$, we rescale the modulated pattern so that
$
\|\tilde e_f\|_2 = \|e\|_2$,
which keeps the total energy of the class-mean difference fixed across the low-, mid-, and high-frequency conditions.

We evaluate these controls in the masked-readout setting with $N=256$, $N'=16$, and $k=16$, sweeping the readout side length $s$ as in Appendix~C-B. Detector noise is fixed to the low-noise setting, with fine-grid noise standard deviation equal to $2\%$ of the mean object intensity per fine-grid pixel. For each frequency condition and readout size, we compare the conventional focusing lens with the jointly optimized phase mask and report both validation accuracy and $d^2$.

\subsection{Additional Dataset Experiment Details}
\label{app:additional_dataset_details}

We evaluate masked-readout trends on MNIST, FashionMNIST, and SVHN.
SVHN is converted to grayscale before simulation. All images are resized
by bilinear interpolation to an $N\times N$ simulator grid, with $N=1024$. We select the coarse-grid resolution to be $N'=4$. We sweep $
s\in\{1,2,4,8,16,32,64,128,256\}, \qquad s\le k$.

For each dataset and each $s$, we compare (i) a fixed conventional lens
and (ii) end-to-end joint optimization of the optical phase and
classifier under the same cross-entropy objective. 

MNIST and FashionMNIST use a grayscale ResNet-18 backbone operating on
the coarse detector measurements. For SVHN, we use a CNN with
approximately $5\times10^5$ parameters and found gradient-based joint-training to be more performant and stable. All models are trained with Adam
(learning rate $10^{-3}$), batch size 64, and test batch size 100 for 3
epochs. Each configuration is repeated over five random seeds
(initialization and minibatch order), and we report mean test accuracy
with $\pm1$ standard deviation.

\begin{figure}[h]
    \centering
    \includegraphics[width=\linewidth]{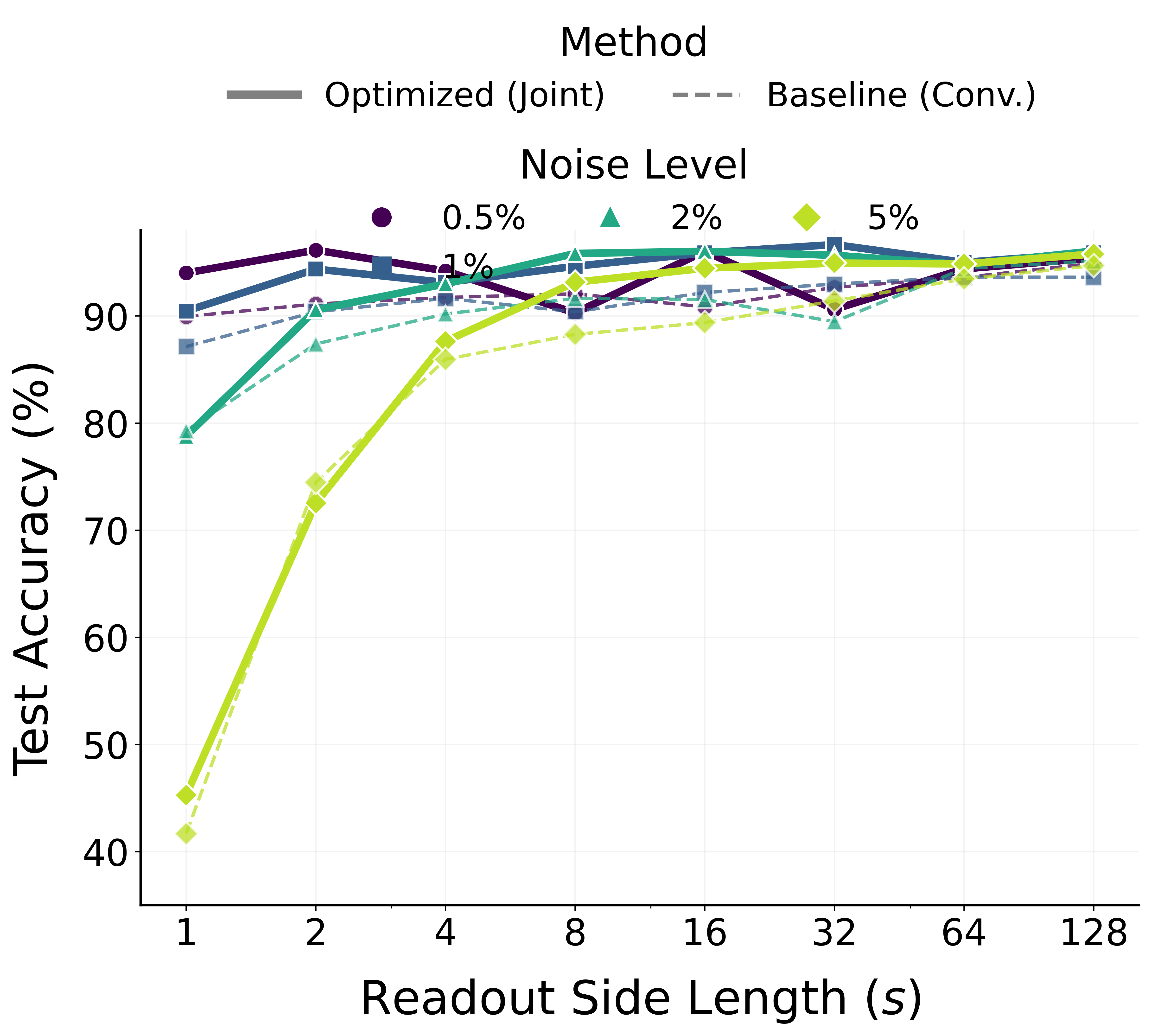}
    \caption{MNIST masked-readout sweep at detector resolution $N'=8$. Accuracy is shown versus readout budget for conventional and jointly optimized systems across detector-noise levels. Compared with the $N'=4$ setting, the performance gap is small, consistent with the conclusion that the benefit of optical co-design decreases as detector resolution increases.}
\label{fig:k128mnist}
\end{figure}
}


 




\vfill

\end{document}


\title{Supplementary Material for\\
``End-to-End Optimization of Incoherent Imaging for Classification Under Detector-Limited Readout''}

\author{Archer Wang, Joshua Chen, Sachin Vaidya, and Marin Solja\v{c}i\'c}

\maketitle

\begin{figure}[!t]
\centering
\subfloat[]{\includegraphics[width=3.3in]{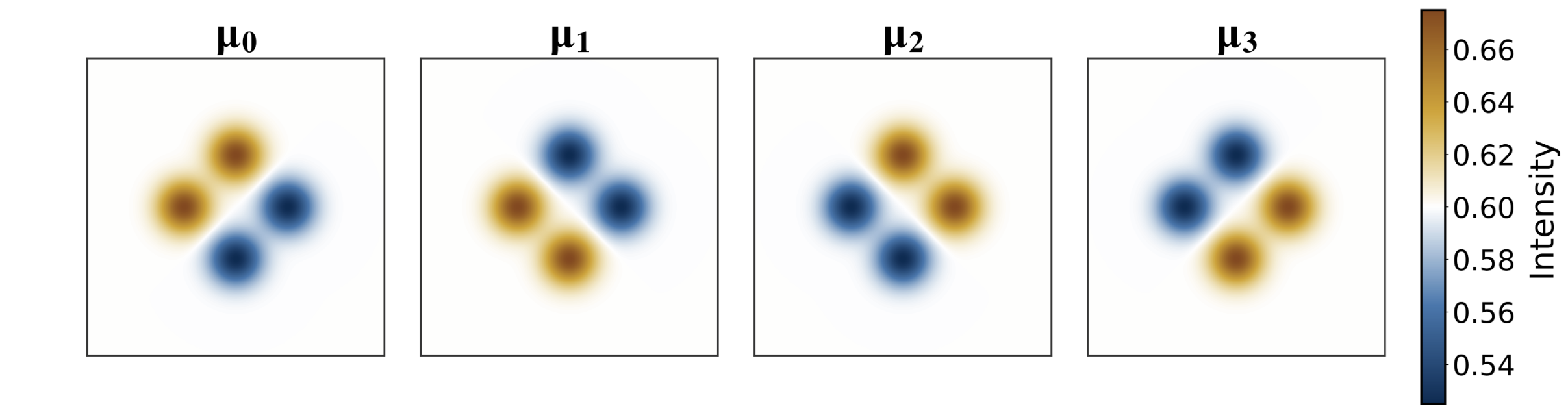}%
\label{fig_first_case}}
\hfil
\subfloat[]{\includegraphics[width=3.3in]{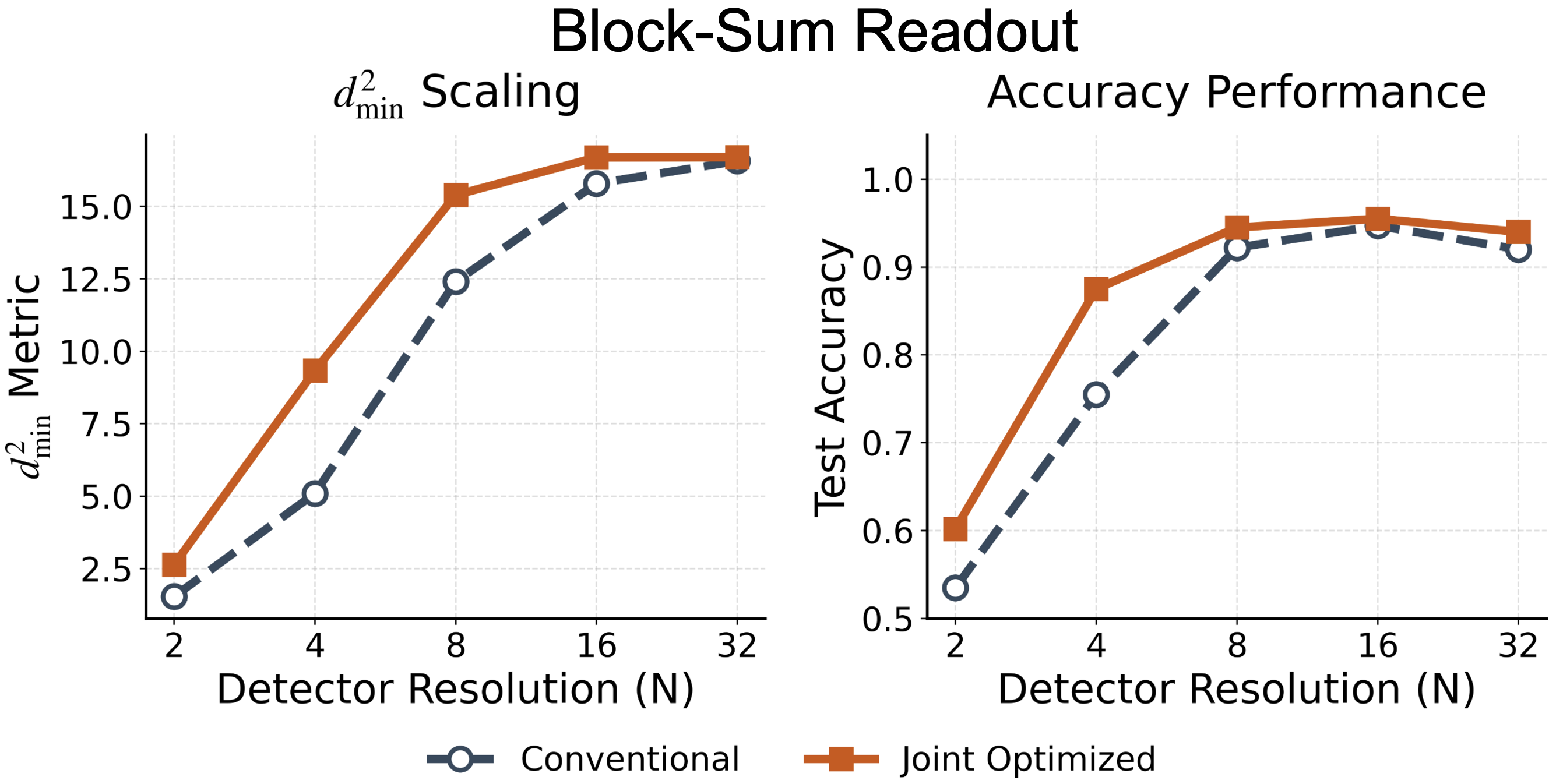}%
\label{fig_second_case}}
\hfil
\subfloat[]{\includegraphics[width=3.3in]{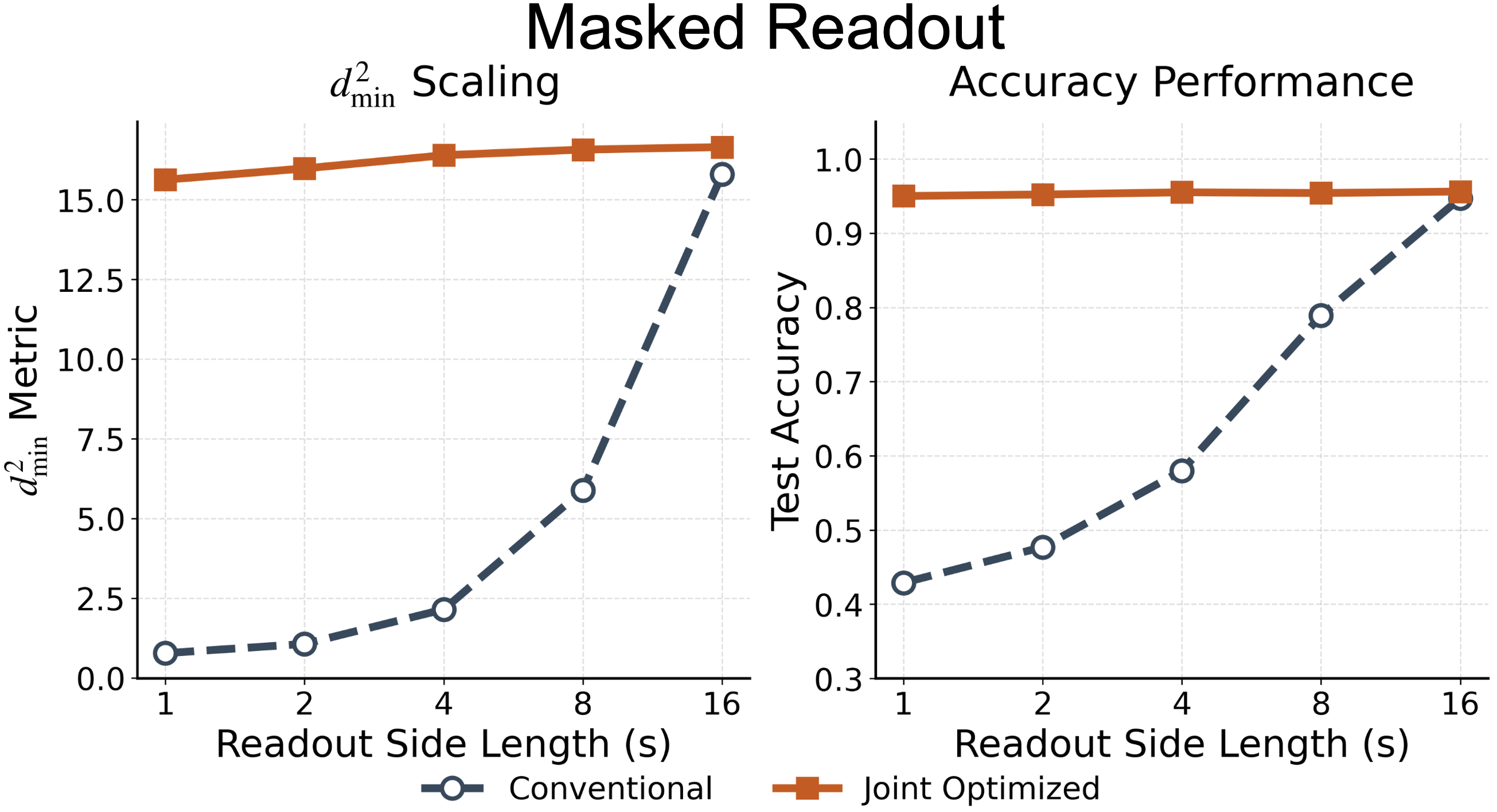}%
\label{fig_second_case}}
\hfil
\subfloat[]{\includegraphics[width=3.3in]{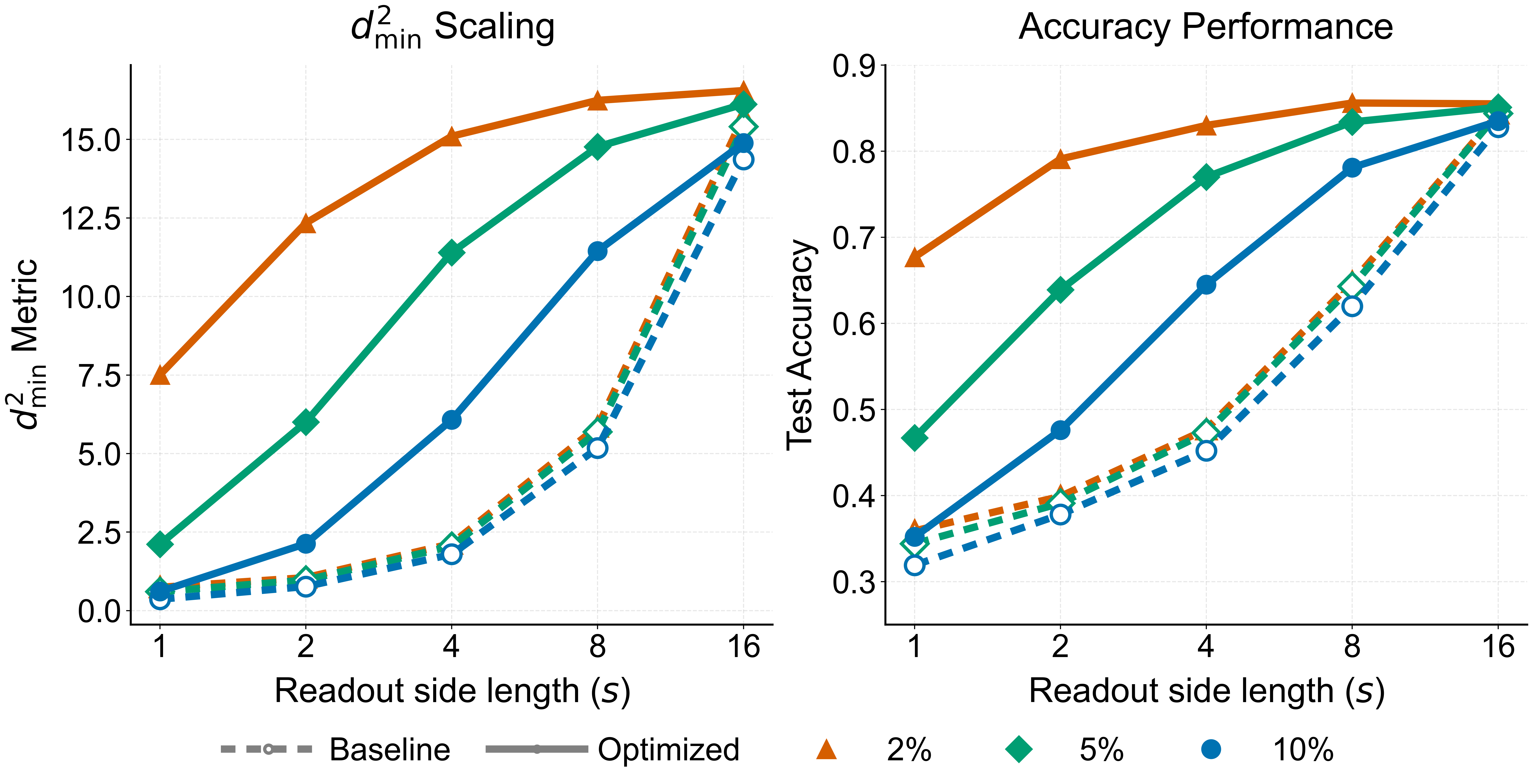}}%
\caption{\textbf{Multiclass ($K=4$) classification under detector-limited readout}.
(a) Four class means in the synthetic Gaussian experiment.
(b) Block-sum readout sweep in the noiseless setting: detector resolution is varied; left, minimum pairwise Mahalanobis distance $d_{\min}^2$; right, test accuracy.
(c) Masked-readout sweep in the noiseless setting: detector resolution is fixed and the within-block readout side length $s$ is varied; left, $d_{\min}^2$; right, test accuracy.
Joint optimization of the optical phase and classifier improves performance over a focusing lens across detector-limited regimes. (d) Consistent with Fig.~4 in the main text (using the same settings as (c)), increasing noise reduces the advantage of joint optimization, especially in the more detector-limited readout regimes.}
\label{fig:k4results}
\end{figure}

\section{Additional Experiment Details}

\subsection{Synthetic Multiclass ($K=4$) Classification Experiment}
\label{sec:toy_m4_supplement}

Now we consider multi-class classification problems, where the samples 
are drawn from $x | C = c \sim N(\mu_c, \Sigma)$ for $c=1,\ldots,K$ 
for $K \ge 3$. We extend our analysis for binary classification to a 
multi-class setting by looking at pairwise $d^2$ distances between 
classes. To verify that the qualitative trends observed in the binary 
case extend beyond two-class discrimination, we consider a four-class 
version of the synthetic Gaussian surrogate. Across our simulations, 
we note the minimum pairwise Mahalanobis distance across all class pairs.

We construct four class means using two approximately orthogonal 
difference-of-Gaussians patterns: a left--right pattern and an up--down 
pattern (Fig.~\ref{fig:k4results}a). Let $(a,b)\in\{\pm1\}^2$ and define 
$\mu_{a,b} = \mu_{\mathrm{base}} + \tfrac12 A\bigl(a\,\mu_{\mathrm{LR}}
+b\,\mu_{\mathrm{UD}}\bigr)$. This yields four classes corresponding to 
the four sign combinations of the two patterns. As shown in 
Fig.~\ref{fig:k4results}b--c, the qualitative trends from the binary 
case carry over to the multiclass setting. As detector resolution 
increases, the absolute performance of the conventional lens baseline 
improves, and the absolute performance of the jointly optimized system 
improves as well. However, because the baseline benefits more strongly 
from increased detector resolution, the performance gap between the two 
systems shrinks.

We use the same covariance model and optical forward model as in the
binary toy experiment. Class means are constructed from left--right and
up--down difference-of-Gaussians templates as described in the main text.
Separability is summarized by the minimum pairwise Mahalanobis distance
over all $\binom{4}{2}$ class pairs.

\section{Additional Theorems and Lemmas}

\begin{lemma}[I--MMSE identity for a scalar AWGN channel \cite{Guo2005}]\label{lem:immse}
Let $X$ be any real-valued random variable with $\mathbb E[X^2]<\infty$, and let $Z\sim\mathcal N(0,1)$ be
independent of $X$. For $\gamma\ge 0$, define the additive white Gaussian noise (AWGN) channel $
Y_\gamma := \sqrt{\gamma}\,X + Z$.
Then the mutual information (in nats) is differentiable in $\gamma$ and satisfies
\[
\frac{d}{d\gamma} I(X;Y_\gamma) \;=\; \frac12\,\mathrm{mmse}(\gamma),
\]
where $
\mathrm{mmse}(\gamma) := \mathbb E\!\left[\bigl(X-\mathbb E[X\mid Y_\gamma]\bigr)^2\right]$.
In particular, $\mathrm{mmse}(\gamma)\ge 0$, so $I(X;Y_\gamma)$ is nondecreasing in $\gamma$; moreover,
if $X$ is non-degenerate (not almost surely constant), then $\mathrm{mmse}(\gamma)>0$ for all finite
$\gamma$, hence $I(X;Y_\gamma)$ is strictly increasing in $\gamma$.
\label{lemma:immse}
\end{lemma}

\begin{theorem}
For binary equiprobable classes with shared covariance, the mutual
information $I(C;Y)$ is a monotonically increasing function of the
squared Mahalanobis distance $d^2$.
\label{thm:monotone}
\end{theorem}

The mutual information takes the closed form
\begin{equation}
    I(C;Y) = \log 2 -
    \mathbb{E}_{T\sim\mathcal{N}\!\left(\frac{1}{2}d^2,\,d^2\right)}
    \!\left[\log\left(1+e^{-T}\right)\right],
\end{equation}

\begin{proof}
Assume two equiprobable classes with shared covariance
\[
Y \mid C=c \sim \mathcal N(\mu_{Y,c},\Sigma_Y), \quad c\in\{0,1\}.
\]

The log-likelihood ratio is
\[
L(y)
= \log\frac{p(y|1)}{p(y|0)}
= (m_1-m_0)^\top \Sigma_Y^{-1}
\!\left(y-\tfrac12(m_0+m_1)\right).
\]

The squared Mahalanobis distance is $
d^2 = \Delta m^\top \Sigma_Y^{-1}\Delta m$ for $\Delta m := m_1-m_0$.

Define the whitened and centered observation
\[
U := \Sigma_Y^{-1/2}\Bigl(Y-\tfrac12(m_0+m_1)\Bigr).
\]

Then
\[
U \mid C=1 = \tfrac12 u + N,
\qquad
U \mid C=0 = -\tfrac12 u + N,
\]
where $N\sim\mathcal N(0,I)$ and
\[
u := \Sigma_Y^{-1/2}(m_1-m_0), \qquad \|u\|^2=d^2 .
\]

The log-likelihood ratio depends on $U$ only through the scalar
projection
\[
S := \frac{u^\top U}{\|u\|},
\]
so $I(C;Y)=I(C;S)$. The induced channel is $
S = \sqrt{\gamma}\,X + W$,
where $X\in\{\pm1\}$ is equiprobable, $W\sim\mathcal N(0,1)$,
and $\gamma=d^2/4$.

By Lemma~\ref{lem:immse}, $
\frac{d}{d\gamma}I(X;S)=\tfrac12\mathrm{mmse}(\gamma)>0$
for all finite $\gamma$. Hence $I(C;Y)$ is strictly increasing in
$d^2$.
\end{proof}

\begin{theorem}[Fano's Inequality \cite{fano1961transmission}]\label{thm:fanos}
Let $X$ be a random variable taking values in a finite set $\mathcal{X}$, and let $\hat{X}$ be an estimate of $X$ based on some observation $Y$. Define the probability of error $
P_e = \mathbb{P}(X \neq \hat{X})$.
Then
\[
H(X \mid Y) \leq h(P_e) + P_e \log(|\mathcal{X}| - 1),
\]
where $h(P_e) = -P_e \log P_e - (1 - P_e)\log(1 - P_e)$ is the binary entropy function.
\end{theorem}
\begin{remark}[Connection from Fano's inequality to mutual information]\label{remark:mutual_info_fano}
Since $
I(X;Y) = H(X) - H(X \mid Y)$,
Fano's inequality implies
\[
I(X;Y) \ge H(X) - h(P_e) - P_e \log\bigl(|\mathcal{X}| - 1\bigr).
\]
Using the looser bound $
H(X \mid Y) \le 1 + P_e \log |\mathcal{X}|$,
we also obtain
\[
P_e \ge \frac{H(X)-I(X;Y)-1}{\log |\mathcal{X}|}.
\]
Thus, if $I(X;Y)$ is small relative to $H(X)$, then the probability of error
cannot be small.
\end{remark}